\documentclass{article} % For LaTeX2e
\usepackage{RobustML_iclr2021_conference_new,times}

\usepackage{amsmath,amsfonts,bm}

\def\eqref#1{equation~\ref{#1}}
\def\1{\bm{1}}

\DeclareMathAlphabet{\mathsfit}{\encodingdefault}{\sfdefault}{m}{sl}
\SetMathAlphabet{\mathsfit}{bold}{\encodingdefault}{\sfdefault}{bx}{n}

\newcommand{\E}{\mathbb{E}}

\newcommand{\R}{\mathbb{R}}

\newcommand{\KL}{D_{\mathrm{KL}}}

\DeclareMathOperator*{\argmin}{arg\,min}

\usepackage[hidelinks]{hyperref}
\usepackage{url}
\usepackage{comment}

\usepackage{times}  % DO NOT CHANGE THIS
\usepackage{helvet} % DO NOT CHANGE THIS
\usepackage{courier}  % DO NOT CHANGE THIS
\usepackage{graphicx} % DO NOT CHANGE THIS
\usepackage{natbib}  % DO NOT CHANGE THIS AND DO NOT ADD ANY OPTIONS TO IT
\usepackage{caption} % DO NOT CHANGE THIS AND DO NOT ADD ANY OPTIONS TO IT
\usepackage{amsmath}
\usepackage[italicdiff]{physics}
\usepackage{mathtools}

\DeclarePairedDelimiter{\diagfences}{(}{)}
\newcommand{\diag}{\operatorname{diag}\diagfences}

\usepackage{amsfonts}
\usepackage{booktabs}
\usepackage{siunitx}
\usepackage{color}
\usepackage{amsthm}
\newtheorem{theorem}{Theorem}
\newtheorem{lemma}[theorem]{Lemma}
\newtheorem{corollary}[theorem]{Corollary}

\usepackage{subfigure}

\usepackage{amsmath}

\usepackage{mathtools}
\newcommand{\MI}{\operatorname{MI}\diagfences}

\usepackage{bbm}

\DeclareMathOperator\supp{supp}

\usepackage{mathtools}
\newcommand\SmallMatrix[1]{{%
		\tiny\arraycolsep=0.3\arraycolsep\ensuremath{\begin{pmatrix}#1\end{pmatrix}}}}
\usepackage{amsthm}
\newtheorem{definition}[theorem]{Definition}

\usepackage{longtable}
\usepackage{booktabs} % For prettier tables
\usepackage{multirow} % Required for multirows

\usepackage{mathrsfs} % https://www.ctan.org/pkg/mathrsfs

\usepackage[ruled,vlined,linesnumbered]{algorithm2e}

\usepackage{wrapfig}
\usepackage{tikz}
\usetikzlibrary{arrows.meta,positioning}

\usepackage{adjustbox}
\usepackage{array}

\newcolumntype{R}[2]{%
    >{\adjustbox{angle=#1,lap=\width-(#2)}\bgroup}%
    l%
    <{\egroup}%
}

\usepackage{enumitem}
\newlist{eqlist}{enumerate*}{1}
\setlist[eqlist]{itemjoin=\quad,mode=unboxed,label=(\roman*),ref=\theequation(\roman*)}

\usepackage{enumitem}% http://ctan.org/pkg/enumitem

\usepackage[normalem]{ulem}

\usetikzlibrary{backgrounds}

\usepackage[para]{footmisc}

\title{Continual Invariant Risk Minimization}
\iclrfinalcopy
\author{Francesco Alesiani 
\thanks{
\url{http://www.neclab.eu}, \texttt{francesco.alesiani@neclab.eu}
} \\
NEC Laboratories Europe\\
Kurfuerstenanlage 36,\\
D-69115 Heidelberg, DE \\
\And
Shujian Yu 
\thanks{
\texttt{s.yu3@vu.nl}
} \\
Computer Science Department \\
Vrije Universiteit of Amsterdam\\
1081 HV Amsterdam, NL\\
\And
Mathias Niepert \\
University of Stuttgart \\
Universitätsstraße 32 \\
D-70569 Stuttgart, DE \\
}

\begin{document}

\maketitle

\begin{abstract}
%Inferring models from observations could be misleading, if proper causal structure is not considered, and leading to poor generalization performance in unseen environments. Invariant risk minimization (IRM) is the principle to discover (anti-) causal relationship from multiple environments, introduced by \cite{arjovsky2019invariant} and extended by \cite{ahuja2020invariant}. In this work we generalize the concept of IRM in more realistic scenarios, where environments are observed sequentially. In this scenario, previous proposed approaches fail to identify the invariant models, including methods for Continual Learning. To this end, we extend IRM under a variational Bayesian framework and additionally proposed a new solution based on the Alternating direction method of multiplier (ADMM) and on an alternative bilevel formulation of IRM. We show with experiment on multiple datasets and different numbers of environments, that the proposed methods outperform previous approaches.

Empirical risk minimization can lead to poor generalization behavior on unseen environments if the learned model does not capture invariant feature representations. Invariant risk minimization (IRM) is a recent proposal for discovering environment-invariant representations. IRM was introduced by \cite{arjovsky2019invariant} and extended by \cite{ahuja2020invariant}. IRM assumes that all environments are available to the learning system at the same time. With this work, we generalize the concept of IRM to scenarios where environments are observed sequentially. We show that existing approaches, including those designed for continual learning, fail to identify the invariant features and models across sequentially presented environments. We extend IRM under a variational Bayesian and bilevel framework, creating a general approach to continual invariant risk minimization. We also describe a strategy to solve the optimization problems using a variant of the alternating direction method of multiplier (ADMM). We show empirically using multiple datasets and with multiple sequential environments that the proposed methods outperform or is competitive with prior approaches.

% We thus propose how to extend the discovery of invariant models with sequential environments and additionally propose a new method based on an alternative formulation of the invariant risk minimization principle. 

% We show with experiment on multiple datasets and number of environments, that the proposed methods outperforms previous approaches. 

% \textcolor{blue}{It has received increased attention to incorporate causal structures in learning models for achieving strong generalization to unseen environment; such structures constrain learning by reducing the effects of spurious features that contradict the causal relation to the target.
% However, learning these structures is a hard problem in itself. Moreover, it's not clear how to incorporate the machinery of causality into a continual learning setup, in which the data from each environment arrives in a sequential manner.}

\end{abstract}

\section{Introduction}

\noindent Empirical risk minimization (ERM) is the predominant principle for designing machine learning models.
In numerous application domains, however, the test data distribution can differ from the training data distribution. For instance, at test time, the same task might be observed in a different environment. Neural networks trained by minimizing ERM objectives over the training
distribution tend to generalize poorly in these situations.
%When sufficient data covering the complete support of the data distribution is available, minimizing ERM has lead to even superhuman performance on various tasks~\cite{serre2019deep}.
%The ERM principle, however, is extremely sensitive to small changes in distribution.
%\mathias{I'm not sure I understand the motivation in the paragraph below. I think you are conflating sensitivity to perturbations (deepfool etc.) with (in the next paragraph) generalization to other (across) environments. This is not (necessarily) the same thing. As far as I see it, we are interested in inviariant representations and, therefore, generalization. I'm not sure we can say anything about robustness to perturbations/attacks.}
%\shujian{I agree, corrected.}
% DeepFool~\cite{moosavi2016deepfool}, for example, has shown that if we have access to the model, we can perturb the test sample in order to predict a completely different class in state-of-the-art methods. 
Improving the generalization of learning systems has become a major research topic in recent years, with many different threads of research including, but not limited to,  robust optimization (e.g.,~\cite{hoffman2018algorithms}) and domain adaptation (e.g.,~\cite{johansson2019support}). Both of these research directions, however, have their own intrinsic limitations~(\cite{ahuja2020invariant}).
%Robust ML has recently addressed the possibility that input distribution is not accurately representing the test input distribution. Although this recent advance, a major modeling problem has not been properly addressed, which also robust ML can not solve.
Recently, there have been proposals of approaches that learn environment-invariant representations. The motivating idea is that if the behavior of a model is invariant across environments, this makes it more likely that the model has captured a causal relationship between features and prediction targets. This in turn should lead to a better generalization behavior. Invariant risk minimization (IRM,~\cite{arjovsky2019invariant}),
which pioneered this idea, introduces a new optimization loss function to identify non-spurious causal feature-target interactions. Invariant risk minimization games (IRMG,~\cite{ahuja2020invariant}) expands on IRM from a game-theoretic perspective. %\textcolor{blue}{The general idea of these methods is that the expected value of $y$ given causal features is constant across all environments.}
%\textbf{PS: can we shorten the following three paragraphs on continual learning?}
%Causal inference models

The assumption of IRM and its extensions, however, is that all environments are available to the learning system at the same time, which is unrealistic in numerous applications. A learning agent experiences environments often sequentially and not concurrently. 
% %
% One often collects observations and then makes decisions or performs reasoning in different environments at different times. \mathias{A concrete example here would be great -- not Newton's apple.} %The falling of the Newton's apple is to be observed in various conditions, but the law does not change.
% One does not need to understand the environment where and when a phenomenon takes place, nor the consequences of the phenomenon, but one ought to isolate the phenomenon from all non-causing elements.
% %In fact, as the most effective method for discovering causal mechanisms from observational data, the scientific method~\cite{sep-scientific-method} also works by continually testing hypotheses and rejecting those that contradict the data.
% %
For instance, in a federated learning scenario with patient medical records, each hospital's (environment) data might be used to train a shared machine learning model that receives the data from these environments in a sequential manner.
The model might then be applied to data from an additional hospital (environment) unavailable at training time. 
%cannot be shared  only share and train the model but {\color{black} In many contexts availability of all the training environments is not possible, for example in medical application, where data is private, and we may only allowed to share models, and coordinated training could not be possible; in other scenarios, the environments are simply not know in advance, for example in robotic application or for autonomous driving, where the robot is required to learned invariant feature representation in various new and unknown conditions. In other situation, retraining is highly inefficient, if not environmental unfriendly, as for example with the recent very large linguistic models (\cite{devlin2018bert}, with up to 0.3G parameters). }
Unfortunately, both IRM and IRMG are incompatible with such a continual learning setup in which the learner receives training data from environments presented sequentially. 
% {\color{black} In fact, IRM requires sampling data from multiple environments simultaneously for computing a regularization term pertinent to its learning objective, where different environments are defined by intervening on one or more variables. Therefore, incorporating the machinery of IRM with that of online continual learning is an open and well-motivated problem.} 
{\color{black} As already noted by \cite{javed2020learning}, ``IRM \cite{arjovsky2019invariant} requires sampling data from multiple environments simultaneously for computing a regularization term pertinent to its learning objective, where different environments are defined by intervening on one or more variables of the world." The same applies to IRMG (\cite{ahuja2020invariant})} 

%Maximum likelihood estimation complex distributions,
%When distributions are complex, simple can not model well these distributions. Bayes inference improves the expressiveness of the model. We follow the variational approach to make bayesian inference efficient \cite{nguyen2018variational}.
%on the contrary help in alleviating this problem. \cite{nguyen2018variational}.
To address the problem of learning environment-invariant ML models in sequential environments, we make the following contributions:
\begin{comment}
\begin{itemize}
\item We provide an alternative formulation of the Invariant Risk Minimization, which we name Bilevel Invariant Risk and provide a partical allgorithm for to solve the BIRM;
\item the introduction of a Bayesian Variational model for the discovery of Invariant model by Invairnat Risk Minimization and a practical algorithm to solve the problem
\item a (single-task multi-environments ) Bayesian Variational Continual Invariant Risk Minimization Learning problem and an algorithm for its solution.
%\item Alternative formulation of IRM (BIRM)
%\item ADMM algorithm of BIRM
\item experimental  comparing the proposed BIRM-ADMM algorithm, Variational Bayes network learning, IRM learning,  Game IRM, ERM and EWC (a Continual Learning algorithm)
%\item term for the complexity of the model (to be done)
\end{itemize}
\end{comment}
\begin{itemize}[nosep, wide, leftmargin=*]
\item We expand both IRM and IRMG under a Bayesian variational framework and develop novel objectives (for the discovery of invariant models) in two scenarios: (1) the standard multi-environment scenario where the learner receives training data from all environments at the same time; and (2) the scenario where data from each environment arrives sequentially.
\item We demonstrate that the resulting bilevel problem objectives have an alternative formulation, which allows us to compute a solution efficiently using the alternating direction method of multipliers (ADMM).
%\item Alternative formulation of IRM (BIRM)
%\item ADMM algorithm of BIRM
\item We compare our method to ERM, IRM, IRMG, and various continual learning methods (EWC, GEM, MER, VCL)   
% (a reference continual learning algorithm) 
on a diverse set of tasks, demonstrating comparable or superior performance in most situations.
%\item term for the complexity of the model (to be done)
\end{itemize}

\section{Background: Offline Invariant Risk Minimization}

We consider a multi-environment setting where, given a set of training environments $E =\{e_1,e_2,\cdots,e_m\}$, the goal is to find parameters $\theta$ that generalize well to unseen (test) environments. Each environment $e$ has an associated training data set $D_e$ and a corresponding risk $R^e$
% \begin{equation}
% R^e(\theta)\doteq E_{(x,y)\sim D_e} \ell_e(f_\theta(x),y).
% \end{equation}
\begin{equation}
R^e(w \circ \phi )\doteq E_{(x,y)\sim D_e} \ell_e((w \circ \phi)(x),y),
\end{equation}
where $f_\theta = w \circ \phi$ is the composition of a feature extraction function $\phi$ and a classifier (or regression function) $w$.
Empirical Risk Minimization (ERM) minimizes the average loss across all training examples, regardless of environment:
\begin{equation}
R_{\text{ERM}}(\theta)\doteq E_{(x,y)\sim \cup_{e\in E} D_e} \ell(f_\theta(x),y).
\end{equation}
ERM has strong theoretical foundations in the case of iid data~(\cite{vapnik1992principles}) but can fail dramatically when test environments differ significantly from training environments.
To remove spurious features from the model, Invariant Risk Minimization (IRM, ~\cite{arjovsky2019invariant}) instead aims to capture invariant representations $\phi$ such that the optimal classifier $w$ given $\phi$ is the same across all training environments. This leads to the following multiple bi-level optimization problem
% This was the two line equation
% \begin{subequations} \label{eq:IRM}
% \begin{eqnarray}
% 	\min_{\phi \in H_\phi,w \in H_w} \sum_{e \in E} R^e(w \circ \phi) \\
% 	\text{s.t.} && w \in \argmin_{w_e \in H_w} R^e(w_e \circ \phi), \forall e \in E,
% \end{eqnarray}
% \end{subequations}
\begin{equation} \label{eq:IRM}
	\min_{\phi \in H_\phi,w \in H_w} ~ \sum_{e \in E} R^e(w \circ \phi) ~~~ \ \ \ 
	\text{s.t. } ~ w \in \argmin_{w_e \in H_w} R^e(w_e \circ \phi), \forall e \in E,
\end{equation}
where $H_\phi,H_w$ are the hypothesis sets for, respectively, feature extractors and classifiers.  
Unfortunately, solving the IRM bi-level programming problem directly is difficult since solving the outer problem requires solving multiple dependent minimization problems jointly. We can, however, relax IRM to IRMv1 by fixing 
% simple linear 
a scalar classifier and learning a representation $\phi$ such that the classifier is ``approximately locally optimal"~(\cite{arjovsky2019invariant})
\begin{eqnarray} \label{eq:IRMv1}
    \min_{\phi \in H_\phi} ~ \sum_{e \in E} R^e(\phi) + \lambda ||\nabla_{w|w=1.0} R^e(w \phi)||^2 , \forall e \in E,
\end{eqnarray}
where $w$ is a scalar evaluated in $1$ and $\lambda$ controls the strength of the penalty term on gradients on $w$.
Alternatively, the recently proposed Invariant Risk Minimization Games (IRMG)~(\cite{ahuja2020invariant}) proposes to learn an ensemble of classifiers with each environment controlling one component of the ensemble. Intuitively, the environments play a game where each environment's action is to decide its contribution to the ensemble aiming to minimize its risk. Specifically, IRMG optimizes the following objective:
% \begin{eqnarray}
% \min_{\phi \in H_\phi,\bar{w} \in H_w} && \sum_{e \in E} R^e(w^{av} \circ \phi) \\
% \text{s.t.} && w_e = \arg \min_{\bar{w}_e \in H_w} R^e \{ \ell (y,\frac1{|E|} (w+w_{-e})\circ \phi \}, \forall e \in E^{\text{tr}}
% \end{eqnarray}
% This was the two line equation
% \begin{subequations} \label{eq:IRMG} 
% \begin{eqnarray}
% \min_{\phi \in H_\phi,\bar{w} \in H_w} && \sum_{e \in E} R^e(w^{av} \circ \phi) \\
% \text{s.t.} && w_e = \argmin_{\bar{w}_e \in H_w} R^e \left( \frac1{|E|} (w+w_{-e})\circ \phi \right), \forall e \in E^{\text{tr}}
% \end{eqnarray}
% \end{subequations}
% \begin{equation} \label{eq:IRMG}
%     \min_{\phi \in H_\phi,\bar{w} \in H_w} \sum_{e \in E} R^e(w^{av} \circ \phi) ~~~
% \text{s.t.} ~ w_e = \argmin_{\bar{w}_e \in H_w} R^e \left( \frac1{|E|} (w+w_{-e})\circ \phi \right), \forall e \in E^{\text{tr}}
% \end{equation}
\begin{equation} \label{eq:IRMG}
    \min_{\phi \in H_\phi} \sum_{e \in E} R^e(\bar{w} \circ \phi) ~~~ \ \ \ 
\text{s.t. } ~ w_e = \argmin_{w \in H_w} R^e \left( \frac1{|E|} (w+w_{-e})\circ \phi \right), \forall e \in E,
\end{equation}
where
$\bar{w} = \frac1{|E|} \sum_{e \in E} w_e$ is the average  and $w_{-e} = \sum_{e' \in E, e' \ne e} w_{e'}$ the complement classifier. 
% These approaches require all the environments to be seen at the same time. 
% $\bar{w} = \frac1{|E|} \sum_{e \in E^{\text{tr}}} w_e$
% $w_{-e} = \sum_{e' \in E^{\text{tr}}, e' \ne e} w_{e'}$ 

% \section{Variational Invariant Risk Minimization}
\section{Continual IRM by Approximate Bayesian Inference}

% \begin{comment}
% \subsection{Multi-sample IRM loss function}
% Based on the IRMG we defined a two sample loss function
% \begin{eqnarray*}
% \min_{\phi} && \E_{ (y,x)_i, (y,x)_j,\sim D \times D } \{ \ell(y_i| w \circ \phi (x_i)) +\ell(y_j|w \circ \phi (x_j)) \}  \\
% && -  \alpha \nabla^T_\phi \ell(y_i| w \circ \phi (x_i)) \nabla_\phi \ell(y_j|w \circ \phi (x_j)) \} \\
% \text{s.t.} && w \in \arg \min_{w'} \E_{(y,x)_i,\sim D} \{ \ell(y_i| w' \circ \phi (x_i)) \} \end{eqnarray*}

% Which describes the objective to learn a classifier (or regressor) that perform on the single sample is common to all samples, based on the learned representation.
% This loss function unifies that continual learnining loss function \cite{javed2019meta,riemer2018learning} and the IRM \cite{arjovsky2019invariant} for two samples.

% %
% %\begin{eqnarray*}
% %	\min_{q} && \E_{ \phi \sim q(\phi)} \{ \ell(y_i| w \circ \phi (x_i)) +\ell(y_j|w \circ \phi (x_j)) \} \\
% %	\text{s.t.}&& w = \arg \min _{q_i} \E_{w' \sim q_i(w)} \{ \ell(y_i| w \circ \phi (x_i)) \}
% %\end{eqnarray*}

% A first approach is to keep an experience memory of the samples on the various environment and then reply the memory and apply the previous defined loss while learning the models.
% \end{comment}

Both IRM and IRMG assume the availability of training data from all environments at the same time, which is impractical and unrealistic in numerous applications. A natural approach would be to combine principles from IRM and continual learning. Experience replay, that is, memorizing examples of past environments and reusing them later, could be possible in some scenarios but it is often difficult to estimate a-priori the extent of replay necessary to achieve satisfactory generalization capabilities. Here, we propose to adopt a probabilistic approach, exploiting the propagation of the model distribution over environments using Bayes' rule. We integrate both IRM and IRMG with stochastic models, introducing their variational counterparts that admit a continual extension. In addition, our approach is justified by the property of the Kullback–Leibler (KL) divergence that promotes invariant distributions when used in sequential learning (as shown in Theorem~\ref{th:i-proj}). 

% {\color{red} Common Approaches in CL: 1) experience reply, but we need to store additional examples, 2) geenrative models, but the are computationale expensize and require accurate generative models, whereas BAyesan method: 1) give richer information on the learned model , 2) using Bayes Rule we can propagate the model sequentially, 3) use of KL divergence promotes invariant models as we show in the theorem } 

% To circumvent this issue, we first expand both IRM and IRMG under a variational Bayesian framework. 
% We then proceed to show that the new objective can be extended to the scenario where environments are given sequentially.
%\mathias{write a few sentences why you are choosing the variational Bayesian framework (justification)}

% (let put this sentence at the end) Due to space limitations, we exemplify our variational extension to IRMG in this section. For IRM, its variational objective can be obtained similarly (see discussion in the supplementary material).

\subsection{Variational Continual Learning}
%In real world application, data samples for the same task comes sequentially or in batches. We consider here the sequential case. We can thus extend the previous proposed problem to the continual learning case. We first note that \cite{nguyen2018variational}

% We now consider sequential environments. 

Following prior work in continual learning (\cite{nguyen2018variational}), let $D_t$ be the training data from the $t$-th environment $e^t$, let $D_1^t$ be the cumulative data up to the $t$-th environment, and let $\theta$ be the parameters of the feature extractor. 
When each environment is given sequentially, we can use Bayes' rule and we have (all proofs are provided in the supplementary material)
\begin{equation}
    p(\theta|D_1^t) \propto p(\theta|D_1^{t-1})p(D_t|\theta),
\end{equation}
that is, once we have the posterior distribution $p(\theta|D_{1}^{t-1})$ at time $t-1$, we can obtain, by applying Bayes rule, the posterior $p(\theta|D_{1}^{t})$ at time $t$ up to a normalization constant. This is achieved by multiplying the previous posterior with the current data likelihood $p(D_t|\theta)$.
% previous sentence:
% which implies that one is able to propagate the posterior $p(\theta|D_{1}^{t-1})$ after observing $(t-1)$ environments to obtain the new posterior $p(\theta|D_1^{t})$. This is possible by multiplying with the likelihood $p(D_t|\theta)$ and normalizing. A continual update of the model emerges from Bayes’ rule. 
% 
% which implies that one is able to propagate the posterior $p(\theta|D_{1}^{t-1})$ after observing $(t-1)$ environments to obtain the new posterior $p(\theta|D_1^{t})$. This is possible by multiplying with the likelihood $p(D_t|\theta)$ and normalizing. A continual update of the model emerges from Bayes’ rule. 
The posterior distribution is in general not tractable and we use an approximation. With the variational approximation, $p(\theta|D_1^t) \approx q_t(\theta)$, it is thus possible to propagate the variational distribution from one environment to the next. From Corollary \ref{th:bayesian} (in the supplementary material) we can write the continual variational Bayesian inference objective as 
\begin{equation}
\label{eq:bayeian_main}
q_t (\theta) = \argmin_{q(\theta)}\E_{(x,y) \sim D_t} \E_{\theta \sim q(\theta)} \{\ell(y,f_\theta (x))\} +
\KL{\left(q (\theta) || q_{t-1}(\theta)\right)},
\end{equation}
from the variational distribution at step $q_{t-1}(\theta)$, with $f_\theta = w\circ\phi$, a function with parameters $\theta$. 

\subsection{Equivalent formulation of IRM as a Bilevel Optimization Problem (BIRM)}
In order to extend the IRM principle of Equation~\ref{eq:IRM} using the principle of approximate Bayesian inference,  
% We first propose thus an alternative definition of the IRM principle. 
% and then propose the extension to variational inference.  
by applying Lemma \ref{th:reform} (in supplementary material), we first introduce the following new equivalent definition of IRM (equation~\ref{eq:IRM}). % principle of Equation~\ref{eq:IRM}.
\begin{definition} [Bilevel IRM  (BIRM)] 
\label{def:BIRM}
Let $H_\phi$ be a set of feature extractors and let $H_w$ be the set of possible classifiers. An {\bf invariant predictor} $w \circ \phi$ on a set of environments $E$ is said to satisfy the Invariant Risk Minimization (IRM) property if it is the solution to the following bi-level Invariant Risk Minimization (BIRM)  problem
\begin{subequations}
\label{eq:BIRM}
\noindent\begin{minipage}[t]{.5\linewidth}
\begin{equation}
\label{eq:BIRMa}
 \min_{\phi \in H_\phi,w \in H_w} \sum_{e \in E} R^e(w \circ \phi) 
\end{equation}
\end{minipage}%
\begin{minipage}[t]{.5\linewidth}
\begin{equation}
\label{eq:BIRMb}
\text{s.t. \ \ }  \nabla_{w} R^e(w \circ \phi)=0, \forall e \in E.
\end{equation}
\end{minipage}
\end{subequations}
\end{definition}
This formulation results from substituting the minimization conditions in the constraint set of the original IRM formulation with the Karush–Kuhn–Tucker (KKT) optimality conditions. This new formulation allows us to introduce efficient solution methods and simplifies the conditions of IRM. It also justifies the IRMv1 model; indeed, when the classifier is a scalar value and the equality constraint is included in the optimization cost function, we obtain Equation~\ref{eq:IRMv1}.
% \subsubsection{Alternative solution method of IRM}
% We have now the possibility to solve BIRM using various approaches. 
% We notice that the alternative definition of BIRM admits new solution approaches. In particular 
To solve the BIRM problem, we propose to use the Alternating Direction Method of Multipliers (ADMM) (\cite{boyd2011distributed}). ADMM is an alternate optimization procedure that improves convergence and exploits the decomposability of the objective function and constraints. %It is now possible to compute the models for each environment in parallel.
Details of the BIRM-ADMM algorithm are presented in the supplementary material. 

\subsection{Bilevel Variational IRM }

At this point, we cannot yet directly extend the IRM principle using variational inference. That is because if we observe all environments at the same time, the prior of the single environment is data independent. Therefore, we substitute $q_{t-1}(\theta)$ from Equation~\ref{eq:bayeian_main} with priors $p_\phi(\theta)$ and $p_w(\omega)$, where $\theta$ and $\omega$ are now the parameters of the two functions $\phi$ and $w$. We also substitute $q_t(\theta)$ with the variational distributions $q_\phi(\theta)$ and $q_w(\omega)$. 
% By observing Definition~\ref{def:BIRM} and Equation~\ref{eq:bayeian_main}, we extend the definition of BIRM.
\begin{definition} [Bilevel Variational IRM (BVIRM)] \label{def:BVIRM}
	Let $P_\phi$ be a family of distributions over feature extractors, and let $P_w$ be a family of distributions over classifiers. A {\bf variational invariant predictor} on a set of environments $E$ is said to satisfy Bilevel Variational Invariant Risk Minimization (BVIRM) if it is the solution to the following problem:
%	\begin{eqnarray}
%		\min_{q_\phi \in P_\phi} && \sum_{e \in E} \E_{w \sim q_w, \phi \sim q_\phi} Q_\phi^e(w \circ \phi) \\
%		\text{s.t.} && \nabla_{q_w}  \E_{w \sim q_w, \phi \sim q_\phi} Q_w^e(w_e \circ \phi) = 0, \forall e \in E
%	\end{eqnarray}
\begin{subequations}
\label{eq:BVIRM}
\noindent\begin{minipage}[t]{.5\linewidth}
\begin{equation}
 \mathop{ \min_{q_\phi \in P_\phi}}_{q_w \in P_w} \ \ \sum_{e \in E} Q_\phi^e(q_w,q_\phi)
\end{equation}
\end{minipage}%
\begin{minipage}[t]{.5\linewidth}
\begin{equation}
\text{s.t.} \ \ \nabla_{q_w}  Q_w^e(q_w,q_\phi) = 0, \forall e \in E,
\end{equation}
\end{minipage}
\end{subequations}
% where, 
%	\begin{eqnarray}
%	 && Q_\phi^e(w \circ \phi) = R^e(w \circ \phi) - \beta KL(q_\phi||p_\phi), \\
%	&& Q_w^e(w \circ \phi) = R^e(w \circ \phi) - \beta KL(q_{w}||p_{w})
%	\end{eqnarray}
% where, \begin{gather} \label{eq:BVIRM}
% Q_\phi^e(q_w,q_\phi) =  \mathop{\E_{w \sim q_w}}_{\phi \sim q_\phi} R^e(w \circ \phi) + \beta \KL(q_\phi||p_\phi), \\
% Q_w^e(q_w,q_\phi) =  \mathop{\E_{w \sim q_w}}_{\phi \sim q_\phi}  R^e(w \circ \phi) + \beta \KL(q_{w}||p_{w})
% \end{gather}
\begin{subequations} \label{eq:BVIRM_q}
	\begin{eqnarray}
	\text{with } && Q_\phi^e(q_w,q_\phi) =  \mathop{\E_{w \sim q_w}}_{\phi \sim q_\phi} R^e(w \circ \phi) + \beta \KL(q_\phi||p_\phi) {\color{black} + \beta \KL(q_{w}||p_{w})}, \\
	\text{and} && Q_w^e(q_w,q_\phi) =  \mathop{\E_{w \sim q_w}}_{\phi \sim q_\phi}  R^e(w \circ \phi) + \beta \KL(q_{w}||p_{w}),
\end{eqnarray}
\end{subequations}
	and where $p_{\phi}$ and $p_{w}$ are the priors of the two distributions. $\beta$ is a hyper-parameter  balancing the ERM and closeness to the prior.%\footnote{A confirmed by the hyper-parameter search, we used $\beta=1$ in the experiment.}. 
\end{definition}
{\color{black} Definition \ref{def:BVIRM} extends Definition \ref{def:BIRM} with the objective of Eq.\ref{eq:bayeian_main}, where the parameters $\phi$ and $w$ are substituted by their distributions $q_\phi$ and $q_w$.}
The gradient of the cost in the inner problem is taken with respect to the distribution $q_w$. When we parameterize $q_\phi$ with $\theta$ and $q_w$ with $\omega$, the gradient is evaluated with respect to these parameters\footnote{Implementation detail using the mean field parameterization and reparametrization trick is provided in the Supplementary Material}, since the condition implies that the solution is locally optimal. If $Q(p,q)$ is convex in the first argument, then the solution is globally optimal. This definition  extends the IRM principle to the case where we use approximate Bayes inference, shaping the variational distributions $q_w$ and $q_\phi$, to be, in expectation, invariant and optimal across  environments. 

% for example with the mean field approximation, we can now use the Generalized ADMM from Lemma \ref{th:gadmm} and substitute
% $f_i(x) \gets Q_\phi^e(q_w,q_\phi) $ and $g_i(x_i) \gets Q_w^e(q_w,q_\phi)$, where the role of $x_i$ is taken by $\omega_e$ and $\theta$ is optimized in an external loop.

%$f_i(x) \gets \E_{w \sim q_w, \phi \sim q_\phi} Q_\phi^e(w \circ \phi)$
%and $g_i(x_i) \gets \E_{w \sim q_w, \phi \sim q_\phi} Q_w^e(w \circ \phi)
%$, where the role of $x_i$ is taken by $\omega_e$ and $\phi$ is optimized in an external loop.

\subsection{The BVIRM ADMM Algorithm}
% When we consider the variational model, we define the V-BIRM-ADMM. 
As noted for the BIRM definition, the solution of the variational BVIRM formulation can be obtained by using ADMM~(\cite{boyd2011distributed}). 
{\color{black} While in general there are no convergence results of ADMM methods for this problem, for local minima,} under proper conditions \footnote{These conditions are specific bounds on the magnitude and variance of the (sub-)gradients of the stochastic function (\cite{ouyang2013stochastic}). We used ELU $\in C^\infty$ in the experiments.}, the stochastic version of ADMM converges with rate $O(1/\sqrt{t})$ for convex functions and $O(\log t /t)$ for strongly convex functions (\cite{ouyang2013stochastic}).
% In order to do this step, we although need the extension of ADMM  \cite{boyd2011distributed} to stochastic problems \cite{ouyang2013stochastic}. 
% % which states that use of ADMM when the cost function is in expectation converges with rate $O(1/\sqrt{t})$.
% \subsubsection{The BVIRM-ADMM Algorithm}
We are now in the position to write the BVIRM-ADMM formulation of the BVIRM problem. 
% Let $.^+$ and $.^-$ be the values of the variable $.$ after and before the update.
{\color{black} ADMM is defined by the update Eq.\ref{eq:bvirm_admm_main}, where we denote with the apexes  $^-$ and $^+$ the value of any variable before and after the update.} 
Moreover, we abbreviate as follows $Q(\omega , \theta) = Q(q(\omega) , q(\theta))$.
\begin{subequations} \label{eq:bvirm_admm_main}
\begin{eqnarray} %\label{eq:bvirm_admm_main}
	\omega_e^+ &=& \argmin_{\omega_e} L_\rho(\omega_e , u_e^-,\omega^-,v_e^-) , \forall e \in E,\\
	\omega^+ &=& 1/|E| \sum_e (\omega_e+u_e) \label{line:sync}\\
	u_e^+ &=& u_e^- + (\omega_e^+ - \omega^+) \\
	v_e^+&=&v_e^- +  \nabla_{q(\omega)} Q_w^e(\omega^+_e , \theta)
\end{eqnarray}
\end{subequations}
% \begin{equation}
%   \label{eq:three}
%   \mbox{%
%   \begin{eqlist}
%   \item\label{eqi:th-one} \(\omega_e^+ = \argmin_{\omega_e} L_\rho(\omega_e , u_e^-,\omega^-,v_e^-) , \forall e \in E\),
%   \item\label{eqi:th-two} \( \omega^+ &=& 1/|E| \sum_i (\omega_e+u_e)\),
%   \item\label{eqi:th-three} \( u_e^+ &=& u_e^- + (\omega_e^+ - \omega^+) \).
%   \end{eqlist}}
% \end{equation}
% \begin{eqnarray} \label{eq:admm4}
% 	L_\rho(w_e,u_e,w,v_e) &=& Q_\phi^e (\omega_e, \theta) \nonumber  \\
% 	&& + \frac{\rho_0}{2}  \|\omega_e - \omega +u_e \|^2 \nonumber  \\
% 	&&  +  \frac{\rho_1}{2}  \|  \nabla_{q(\omega)} Q_w^e(\omega_e \circ \phi) + v_e\|^2
% \end{eqnarray}
with
\begin{eqnarray} 
\label{eq:bvirm_admm_lagrangian}
\hspace{-5mm} L_\rho(w_e,u_e,w,v_e) & = & Q_\phi^e (\omega_e, \theta) + \frac{\rho_0}{2}  \|\omega_e - \omega +u_e \|^2  +  \frac{\rho_1}{2}  \|  \nabla_{q(\omega)} Q_w^e(\omega_e \circ \phi) + v_e\|^2.
\end{eqnarray}
Here, $\phi$ is fixed and $\theta$ is updated in an external loop or given (e.g. the identity function). In the experiment we use stochastic Gradient Descent (SGD) to update both the model parameters $w_e$ and the feature extractor parameters $\phi$. The result follows by applying Lemma \ref{th:gadmm} in the supplementary material and substituting $x_i \gets \SmallMatrix{w_e \\ \phi }$, $ f_i(x_i) \gets Q_\phi^e (\omega_e, \theta)$ and $g_i(x_i) \gets \nabla_{q(\omega)} Q_w^e(\omega^+_e , \theta)$.
We provide a pseudo-code implementation leveraging  Equation~\ref{eq:bvirm_admm_main} as Algorithm~\ref{alg:BVIRM-ADMM}.
{\color{black}
One of the advantages of the ADMM formulation of BVIRM of Eq.\ref{eq:bvirm_admm_main}, is that it can be computed in parallel, where only Eq.\ref{line:sync} requires synchronization among environments, while the other steps can be computed independently.
}

% A special attention is to the gradient. Indeed the distribution is parameterized by the $\omega$. In this formulation we fix $\phi$, but the approx applies also if we use the variable $x_i \gets \SmallMatrix{w_e \\ \phi }$ in Lemma \ref{th:gadmm} or $w_e \gets \SmallMatrix{w_e \\ \phi }$ in BVIRM-ADMM algotihm.

% \subsubsection{BVIRM-ADMM}

% We write now the Mean Field BVIRM-ADMM method, where the expectation is evaluated using Monte Carlo sampling and the gradient of the KL divergence can be computed in closed form
% \begin{eqnarray} \label{eq:admm5}
% 	\omega_e^+ &=& \arg \min_{\omega_e} L_\rho(\omega_e , u_e^-,\omega^-,v_e^-) , \forall e \in E\\
% 	\omega^+ &=& 1/|E| \sum_i (\omega_e+u_e)\\
% 	u_e^+ &=& u_e^- + (\omega_e^+ - \omega^+) \\
% 	v_e^+&=&v_e^- +  \nabla_{\omega} Q_w^e(\omega^+_e , \theta)
% \end{eqnarray}	
% where
% % \begin{eqnarray} \label{eq:admm6}
% % 	L_\rho(w_e,u_e,w,v_e) &=& Q_\phi^e (\omega_e, \theta) \nonumber  \\
% % 	&& + \frac{\rho_0}{2}  \|\omega_e - \omega +u_e \|^2 \nonumber  \\
% % 	&&  +  \frac{\rho_1}{2}  \|  \nabla_{\omega} Q_w^e(\omega_e \circ \phi) + v_e\|^2
% % \end{eqnarray}

% \begin{eqnarray} \label{eq:admm6}
% 	L_\rho(w_e,u_e,w,v_e) &=& Q_\phi^e (\omega_e, \theta) \nonumber  + \frac{\rho_0}{2}  \|\omega_e - \omega +u_e \|^2  +  \frac{\rho_1}{2}  \|  \nabla_{\omega} Q_w^e(\omega_e \circ \phi) + v_e\|^2
% \end{eqnarray}

\begin{minipage}[t]{0.48\textwidth}
\begin{algorithm}[H]
	\label{alg:BVIRM-ADMM}
	\SetAlgoLined
	\KwResult{$w \circ \phi$ : feature extraction and classifier for the environment $E$ }
	\tcp{Randomly initialize the variables}
	$\omega, \omega_e, u_e, v_e, \theta \gets \text{Init()} $ \;
	%	$\phi \gets \text{Init()} $ \;
	\tcp{Outer loop (on $\theta$) and Inner loop (on $\omega$) }
	\While{not converged}{
% 		\tcp{ Update $\phi$ using stochastic gradient descent(SGD) }
        \tcp{ Update $\phi$ using SGD}
% 		$\theta=  \text{SGD}_{\theta} (\sum_e \E_{w \sim q_\omega, \phi \sim q_\theta} R^e (w_\omega \circ \phi_\theta ))  - \beta \KL ({q_\theta||p_\phi)} $ \;
        % $\theta=  \text{SGD}_{\theta} (\sum_e \E_{w \sim q_\omega, \phi \sim q_\theta} R^e (w_\omega \circ \phi_\theta ))  - \beta \KL ({q_\theta||p_\phi)} $ \;
        $\theta=  \text{SGD}_{\theta} (\sum_{e \in E} Q_\phi^e(q_w,q_\phi))$ \;
		\For {$k = 1, \dots, K$}{
			\For {$e \in E$}{
				$\omega_e = \text{SGD}_{\omega_e} L_\rho(\omega_e , u_e, \omega, v_e) $ \;
				$\omega = 1/|E| \sum_e (\omega_e+u_e)$ \;
				$u_e = u_e + (\omega_e - \omega)$ \;
				% $Q^e(\omega_e , \theta) = \frac1{N} \sum_{i=1}^N R(w_i \circ \phi_i) - \beta \KL{(q(\omega_e^+)||p)}$ \;
				% with \;
				% $w_i \sim \mu(\omega_\mu) + \epsilon \sigma(\omega_\sigma) $\;
				% $\phi_i \sim \mu(\theta_\mu) + \epsilon \sigma(\theta_\sigma) $\;
				% $\epsilon \sim N(0,1)$ \;
				$v_e = v_e +  \nabla_{\omega} Q^e(\omega_e,  \theta)$ \;
			}
		}
	}
	\vspace{4mm}
	\caption{$w,\phi \gets $ BVIRM-ADMM($E,R^e$) ADMM version of the Bilevel Variational IRM Algorithm}
\end{algorithm}
\end{minipage}
\hfill
\begin{minipage}[t]{0.48\textwidth}
\begin{algorithm}[H]
	\label{alg:CL-BVIRM-ADMM}
	\SetAlgoLined
	\KwResult{$w_\omega \circ \phi-\theta$ : feature extraction and classifier for the environment $E$ }
	\tcp{Randomly initialize the variables}
	$\omega, \omega_e, u_e, v_e, \theta \gets \text{Init()} $ \;
	$\bar{\omega} = 0$ \;
	%	$\phi \gets \text{Init()} $ \;
	%	\tcp{Outer (on $\theta$) and Inner loop (on $\omega$) }
	\For {$e \in E$}{
		\For {$k = 1, \dots, K$}{
	    $\theta=  \text{SGD}_{\theta} (Q_\phi^e(q_w,q_\phi))$ \;
		\While{not converged}{
			\tcp{ Update $\omega$ using SGD and ADMM}
			%			$\theta=  \text{SGD}_{\theta} (\sum_e R^e (w_\omega \circ \phi_\theta ))  $ \;
			%			\For {$k = 1, \dots, K$}{			
			$\omega_e = \text{SGD}_{\omega_e} L_\rho(\omega_e , u_e, \omega, v_e) $ \;
			$\omega = 1/2 (\omega_e+u_e + \bar{\omega})$ \;
			$u_e = u_e + (\omega_e - \omega)$ \;
			%					$Q^e(\omega_e , \theta) = \frac1{N}  \sum_{i=1}^N R(w_i \circ \phi_i) -\beta \KL{q(\omega_e^+)||p}$ \;
			%%					with $w_i \sim \mu(\omega_\mu) + \epsilon \sigma(\omega_\sigma) $,  $\epsilon \sim N(0,1)$ \;
			%					with \;
			%					$w_i \sim \mu(\omega_\mu) + \epsilon \sigma(\omega_\sigma) $\;
			%					$\phi_i \sim \mu(\theta_\mu) + \epsilon \sigma(\theta_\sigma) $\;
			%					$\epsilon \sim N(0,1)$ \;
			$v_e = v_e +  \nabla_{\omega} Q^e(\omega_e,  \theta)$ \;
			%				}
		}
		}
		$\bar{\omega} = \omega_e$ \;
	}
	\caption{$w,\phi \gets $ C-BVIRM-ADMM($E,R^e$) ADMM version of the Bilevel Variational IRM Algorithm}
\end{algorithm}
\end{minipage}

\medskip

\subsubsection{The Continual BVIRM ADMM Algorithm}

% We are now in the position to extend BVIRM to the scenario where the environments are observed sequentially. Integrating the definition of BVIRM Eqs.~(\ref{eq:BVIRM}) with continual Bayesian learning (Equation~(\ref{eq:bayeian_main})) we obtain the variational objective of BVIRM for the scenario of sequentially given environments. 
{\color{black} 
In presence of sequential environments, the priors for the new environment are given by the previous environment's distributions $q^-_\phi$ and $q^-_w$, this is obtained by comparing the BVIRM definition in Eqs.~(\ref{eq:BVIRM}) with the continual Bayesian learning Equation~(\ref{eq:bayeian_main}).
}
In Equation~\ref{eq:BVIRM_q} we thus now have $ Q_\phi^e(q_w,q_\phi) =  \mathop{\E}_{w \sim q_w, \phi \sim q_\phi} R^e(w \circ \phi) + \beta \KL(q_\phi||q^-_\phi) {\color{black}  + \beta \KL(q_{w}||q^-_{w}) }$ and $Q_w^e(q_w,q_\phi) =  \mathop{\E}_{w \sim q_w, \phi \sim q_\phi}  R^e(w \circ \phi) + \beta \KL(q_{w}||q^-_{w})$
% , where $q^-_\phi$ and $q^-_w$ are the distributions of the previous environment. 
Algorithm~\ref{alg:CL-BVIRM-ADMM} presents an example implementation of ADMM\footnote{{\color{black} In Algorithm 1 the ADMM update equation is implemented from line $6$ to line $9$, while in Algorithm 2, from line $7$ to line $10$.}} applied to the continual BVIRM formulation. 

\subsection{Information-Theoretic interpretation of C-BVIRM}

\begin{wrapfigure}[12]{R}{0.3\textwidth}
\vspace{-4mm}
\centering
\includegraphics[width=0.27\textwidth]{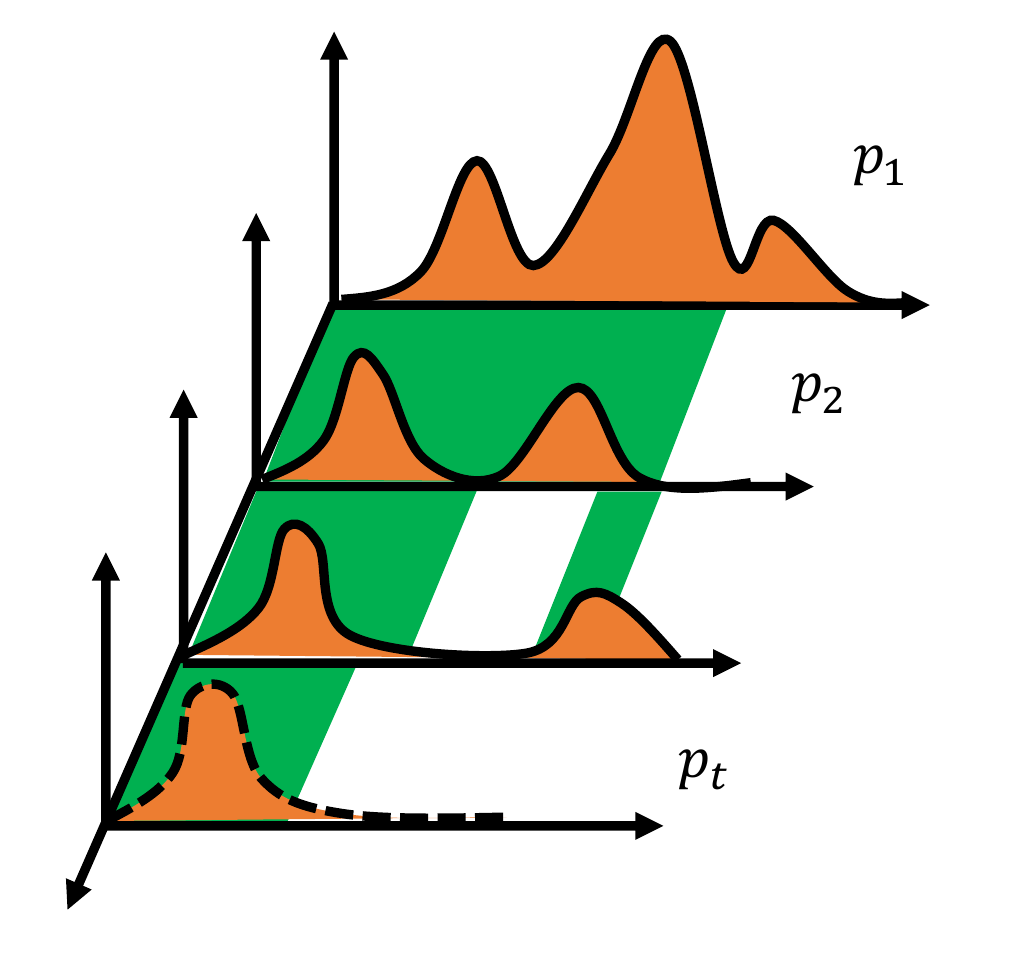}
% \caption{\label{fig:invariant_distribution} Visualization of the sequential projection of distribution when observing new environments}
\caption{\label{fig:invariant_distribution} Sequential projection of distributions $p_1,\dots p_t$, where { \color{black}  $p_{i+1}=\arg \min_{p \in P_{i+1}} \KL(p||p_{i})$ } }
\end{wrapfigure}
%  Kullback–Leibler (KL)
The KL divergence provides an additional motivation for the methods we propose. Indeed, for causal discovery \cite{peters2015causal} suggests a discovery mechanism for causal variables as the intersection of the invariant conditional distributions across environments subject to interventions. The KL divergence is asymmetric and only components present in the first argument distribution are evaluated. This implies that by using the KL divergence we can compute the intersection of the distributions, even when these are observed sequentially. This can be made more explicit by the property of the information projection \cite{cover1999elements}
\begin{theorem} [Information Projection] \label{th:i-proj}
If $P$ and $Q$ are two families of distributions with partially overlapping support, $ \emptyset \subset \supp (P) \bigcap \supp (Q) $, and $q \in Q$, then 
$$ p^* = \argmin_{p \in P} \KL(p||q)  \ \ \ \ \ \ \ \ \ \ \ \ \ \ \ \ \ \ \ \ \ \ \ \ \ \ \ \ \ \ \ \ \ \ \ \ \ \ \ \ $$
has support in the intersection for the support of $P$ and $q$, or $\supp (p^*) \subseteq \supp (P) \bigcap \supp (q) $.
\end{theorem}
Therefore, if we have a sequence of sets of distributions of models from intervention environments and we compute the projection in sequence, the final projected distribution has support on the intersection of all previous distribution families, or $\supp (P_t) = \bigcap_{i=1}^{t} \supp (P_i)$ (see Figure~\ref{fig:invariant_distribution})
{\color{black}, since at each step $p_{i+1}=\arg \min_{p \in P_{i+1}} \KL(p||p_{i})$ }.

% If the environment are seen sequentially, then we can not update the model of the previous environment. We consider here first the case where we do not use episodic memory of the previous environment, whereas this approach would improve the performance.
% Since we consider the single task, the feature extraction model is considered constant (or identity) and will not be updated.

\section{Related Work}

\paragraph{Generalization} 
Domain adaptation~\citep{ben2007analysis,johansson2019support} aims to learn invariant features or components $\phi(x)$ that have similar $P(\phi(x))$ on different (but related) domains by explicitly minimizing a distribution discrepancy measure, such as the Maximum Mean Discrepancy (MMD)~\citep{gretton2012kernel} or the Correlation Alignment (CORAL)~\citep{sun2016deep}. The above condition, however, is not sufficient to guarantee successful generalization to unseen domains, even when the class-conditional distributions of all covariates changes between source and target domains~\citep{gong2016domain,zhao2019learning}. Robust optimization~\citep{hoffman2018algorithms,lee2018minimax}, on the other hand, minimizes the worst performance over a set of possible environments $E$, that is, $\max_{e\in E} R^e(\theta)$. This approach usually poses strong constraint on the closeness between training and test distributions~\citep{bagnell2005robust} which is often violated in practical settings~\citep{arjovsky2019invariant,ahuja2020invariant}.

%Given a single training environment, a set of possible environments is often generated using some set of allowed perturbations, $e.g.$, $P^e(x)$ varying within a KL-divergence $\epsilon$-ball of the empirical distribution~\cite{bagnell2005robust}. However, this assumption does not always hold true for test distributions~\cite{ahuja2020invariant}.

Incorporating the machinery of causality into learning models is a recent trend for improving generalization. \citep{bengio2019meta} argued that causal models can adapt to sparse distributional changes quickly and proposed a meta-learning objective
that optimizes for fast adaptation. IRM, on the other hand, presents an optimization-based formulation to find non-spurious actual causal factors to target $y$. Extensions of IRM include IRMG and the Risk Extrapolation (REx)~\citep{krueger2020out}.
%Other extensions include the Risk Extrapolation (REx)~\cite{krueger2020out} and the online causal learning models~\cite{javed2020learning}. 
% 
Our work's motivation is similar to that of  online causal learning~\citep{javed2020learning}, which models the expected value of target $y$ given each feature as a Markov decision process (MDP) and identifies the spurious feature $x_i$ if $\mathbb{E}[y|x_i]$ is not consistent to temporally distant parts of the MDP. The learning is implemented with a gating model and behaves as a feature selection mechanism and, therefore, can be seen as learning the support of the invariant model. The proposed solution, however, is only applicable to binary features and assumes that the aspect of the spurious variables is known (e.g. the color). It also requires careful hyper-parameter tuning. 
{\color{black}
In the cases where data is not divided into environments, Environment Inference for Invariant Learning (EIIL) classification method (\cite{creager2020environment}) aims at splitting the samples into environments. This method proves to be effective also when the environment label is present. 
}

\paragraph{Continual Learning}
\cite{kirkpatrick2017overcoming,de2019continual} addresses the problem of learning one classifier that performs well across multiple tasks given in a sequential manner. The focus is on the avoidance of catastrophic forgetting. 
With our work, we shift the focus of continual learning to the study of a single task that is observed in different environments.

\section{Experimental Evaluation}

\begin{figure}[!t]
	\centering
	\includegraphics[width=1.
% 	\textwidth]{ds_16.pdf}
    \textwidth]{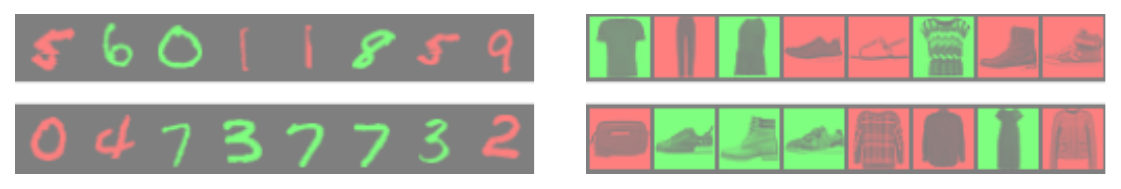}
	\caption{The two color models (on the left {\tt b01}, on the right {\tt b11}) for the train (upper row) and test (lower row) of the MNIST (left) and FashionMNIST (right) datasets.}  
	\label{fig:color_correlation}
\end{figure}

\subsection{Datasets and Experiment Setup}

% \begin{figure}[!hbpt]
% 	\centering
% 	\subfigure[] {
% 		\includegraphics[width=0.4\textwidth]{dataset_pt_MNIST_type_b01_train_ver_vis.pdf}
% 	}
% 	\subfigure[] {
% 		\includegraphics[width=0.4\textwidth]{dataset_pt_MNIST_type_b01_test_ver_vis.pdf}
% 	}
% 	\caption{MNIST dataset training (a) and testing (b) environments; the color is inverted as an example, when the number changes color}
% 	\label{fig:mnist_train_test_b01}
% \end{figure}

% \begin{figure}[!hbpt]
% 	\centering
% 	\subfigure[] {
% 		\includegraphics[width=0.4\textwidth]{dataset_pt_FashionMNIST_type_b11_train_ver_vis.pdf}
% 	}
% 	\subfigure[] {
% 		\includegraphics[width=0.4\textwidth]{dataset_pt_FashionMNIST_type_b11_test_ver_vis.pdf}
% 	}
% 	\caption{Fashion MNIST dataset training (a) and testing (b) environments; the color is inverted as an example; in this case the background changes color}
% 	\label{fig:mnist_train_test_b11}
% \end{figure}

\paragraph{Colored MNIST}
\begin{wrapfigure}[11]{R}{0.24\textwidth}
\vspace{-4mm}
\centering
% \begin{tikzpicture}[
%       mycircle/.style={
%          circle,
%          draw=black,
%          fill=gray,
%          fill opacity = 0.3,
%          text opacity=1,
%          inner sep=0pt,
%          minimum size=20pt,
%          font=\small},
%       myarrow/.style={-Stealth},
%       node distance=0.6cm and 1.2cm
%       ]
%       \node[mycircle] (c1) {Digit};
%       \node[mycircle,above =of c1] (c2) {Color};
%       \node[mycircle,right=of c1] (c3) {Image};
%     \foreach \i/\j/\txt/\p in {% start node/end node/text/position
%       c1/c3//,
%       c2/c3//
%       }
%       \draw [myarrow] (\i) -- node[sloped,font=\small,\p] {\txt} (\j);
%     \end{tikzpicture} 
% \begin{tikzpicture}[
%       mycircle/.style={
%          circle,
%          draw=black,
%          fill=gray,
%          fill opacity = 0.3,
%          text opacity=1,
%          inner sep=0pt,
%          minimum size=20pt,
%          font=\small},
%       myarrow/.style={-Stealth},
%       node distance=0.5cm and .7cm
%       ]
%       \node[mycircle] (digit) {Label};
%       \node[mycircle,above =of digit] (color) {Color};
%       \node[mycircle,right=of digit] (image) {Image};
%     %   \node[inner sep=0pt, above left=of c2] (hammer) {A};
%       \node[inner sep=0pt, above left = .01cm of color ] (hammer) {\includegraphics[width=.05\textwidth]{hammer.pdf}};
%     \foreach \i/\j/\txt/\p in {% start node/end node/text/position
%       digit/image/a/,
%       color/image/b/,
%       digit/color/c/
%       }
%       \draw [myarrow] (\i) -- node[sloped,font=\small,\p] {\txt} (\j);
%     \end{tikzpicture} 
    \begin{tikzpicture}[
      mycircle/.style={
         circle,
         draw=black,
         fill=white,
         fill opacity = 0.3,
         text opacity=1,
         inner sep=0pt,
         minimum size=20pt,
         font=\small},
      myarrow/.style={-Stealth},
      node distance=0.6cm and 1.2cm
      ]
      \node[mycircle] (digit) {Digit};
      \node[mycircle,above =of digit] (color) {Color};
      \node[mycircle,right=of digit] (image) {Image};
    %   \node[inner sep=0pt, above left=of c2] (hammer) {A};
      \node[inner sep=0pt, above left = .1cm of color ] (hammer) {\includegraphics[width=.05\textwidth]{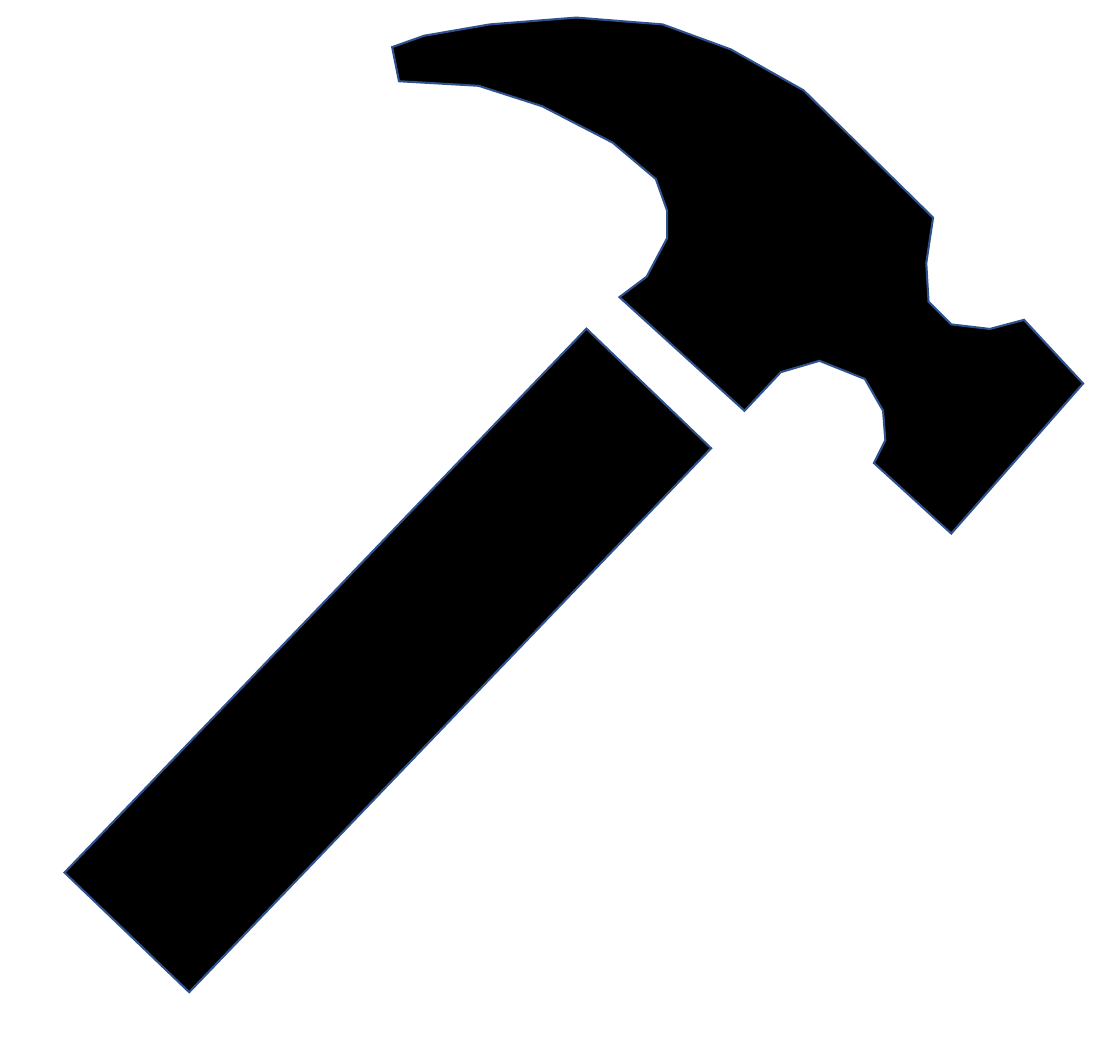}};
    \foreach \i/\j in {% start node/end node/text/position
      digit/image/,
      color/image/,
      digit/color/
      }
      \draw [myarrow] (\i) -- node {} (\j);
    \end{tikzpicture} 

    \caption{\label{fig:experiment_causal_model} Causal relationships of colored MNIST.}
\end{wrapfigure}
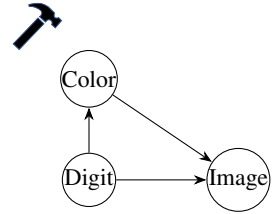
% We consider an extension of the synthetic dataset proposed in \cite{arjovsky2019invariant, ahuja2020invariant}. The task is to classify if the picture is even or odd. $|E|$ environments are created from the original dataset. Figure~\ref{fig:mnist_train_test_b01} 
Figure~\ref{fig:color_correlation} (left) shows a sample of train (upper) and test (lower) samples. In each training environment,
the task is to classify whether the digit is, respectively, even or odd. As in prior work, we add noise to the preliminary label by randomly flipping it with a probability of $0.25$. The color of the image is defined by the variable $z$, which is the noisy label flipped with probability $p_c \in [0.1,0.2]$.  The color of the digit is green if $z$ is even and red if $z$ is odd. Each train environment contains $30,000$ images of size $28 \times 28$ pixels, while the test environment contains $10,000$ images where the probability $p_c = 0.9$. The color of the digit ({\tt b01}) or the background ({\tt b11}) is thus generated from the label but depends on the environment. Figure~\ref{fig:experiment_causal_model} depicts the causal graph (the hammer indicating the effect of the intervention) of the environment. The variable ``Color" is inverted when moving from the training to test environment.

\paragraph{Colored FashionMNIST, KMNIST, and  EMNIST}
%Figure~\ref{fig:color_correlation}(right) the Fashion MNIST dataset \cite{xiao2017fashion} is shown, where the variable $z$ defines the background color and labels are $y=0$ with even classes and $y=1$ with odd ones.
Figure~\ref{fig:color_correlation} (right) shows the Fashion-MNIST dataset, where the variable $z$ defines the background color. Again, we add noise to the preliminary label ($y = 0$ for ``t-shirt", ``pullover", ``coat", ``shirt", ``bag" and $y = 1$ for ``trouser", ``dress", ``sandal", ``sneaker", ``ankle boots") by flipping it with $25$ percent probability to construct the final label.
Besides, we also consider Kuzushiji-MNIST dataset~\cite{clanuwat2018deep}\footnote{\url{https://github.com/rois-codh/kmnist}} and the EMNIST Letters dataset~\cite{cohen2017emnist}\footnote{\url{https://www.nist.gov/itl/products-and-services/emnist-dataset}}. The former includes $10$ symbols of Hiragana, whereas the latter contains $26$ letters in the modern English alphabet. For EMNIST, there are $62,400$ training samples per environment and $20,300$ test samples. We set $y=0$ for letters `a', `c', `e', `g', `i', `k', `m', `o', `q', `s', `u', `v', `y' and $y=1$ for remaining ones.

\textbf{Reference Methods.} We compare with a set of popular reference methods in order to show the advantage of the variational Bayesian framework in learning invariant models in the sequential environment setup. For completeness, we also evaluate the performances of four reference continual learning methods. These include Elastic Weight Consolidation (EWC,~\cite{kirkpatrick2017overcoming}), Gradient Episodic Memory (GEM, ~\citep{lopez-paz_gradient_2017})\footnote{\url{https://github.com/facebookresearch/GradientEpisodicMemory}}, Meta-Experience Replay (MER, ~\cite{riemer2018learning})\footnote{\url{https://github.com/mattriemer/mer}}, and Variational Continual Learning (VCL,~\cite{swaroop2019improving, nguyen2018variational})\footnote{\url{https://github.com/nvcuong/variational-continual-learning}}.
\textbf{ERM} is the classical empirical risk minimization method; we always use the cross-entropy loss. 
\textbf{IRMv1} enforces the gradient of the model with respect to a scalar to be zero.
\textbf{IRMG} models the problem as a game among environments, where each environment learns a separate model.
\textbf{EWC} imposes a regularization cost on the parameters that are relevant to the previous task, where the relevance is measured by Fisher Information (FI); \textbf{GEM} uses episodic memory and computes the updates such that accuracy on previous tasks is not reduced, using gradients stored from previous tasks; 
\textbf{MER} uses an efficient replay memory and employs the meta-learning gradient update to obtain a smooth adaptation among tasks;
% \textbf{VCL*} is a modified version of the original VCL, 
\textbf{VCL} and Variational Continual Learining with coreset \textbf{VCLC} apply variational inference to continual learning.  
% modified version of the original VCL, 
%which shares similar motivation of our approach, but applied to CL.
\textbf{C-VIRMv1} and \textbf{C-VIRMG} refer to, respectively, our proposed variational extensions of \textbf{IRMv1} and \textbf{IRGM} in sequential environments. \textbf{C-BVIRM} is the implementation with ADMM.

All hyper-paramter optimization strategies and simulation configurations are discussed in detail in the supplementary material.

% implements \textbf{C-BVIRM} with ADMM. 
%The method uses also spurious correlation in order to improve prediction accuracy, but do not generalize in other environments.

% \begin{figure}[!hbpt]
% 	\centering
% % 	\includegraphics[width=1.0\textwidth]{horizontal_16.2.pdf}
% 	\includegraphics[width=1.0\textwidth]{horizontal_16.4.pdf}
% 	\caption{The mean performance (over $5$ repetitions)  in training and testing of the proposed methods compared with existing proposed methods for $2$ consecutive environments on MNIST and the b01 color correlation. }
% 	\label{fig:MNIST2}
% \end{figure}

\newcommand*\rot{\multicolumn{1}{R{60}{1em}}}% no optional argument here, please!

% \begin{tabular}{r|ccc}
% &
% \rot{Property 1} &
% \rot{Property 2} &
% \rot{Property 3}
%     \\ \hline
% System 1        &       &       &  X    \\ 
% System 2        & X     & X     &  X    \\
% System 3        & X &   &  X    \\ \hline
% \end{tabular}

\begin{table}
\centering
	\caption{Mean accuracy ($N=5$) on train and test environments when training on $2$ consecutive environments on MNIST and the b01 color correlation.}
	\label{tab:MNIST2}
%sorry i am done	
% 	\small
\footnotesize
\begin{tabular}{rlllllllllll}
\toprule
&\rot{C-BVIRM} &  \rot{C-VIRMG} &  \rot{C-VIRMv1} &  \rot{ERM} &  \rot{EWC} &  \rot{GEM} &  \rot{IRMG} &  \rot{IRMv1} &  \rot{MER} &  \rot{VCL} &  \rot{VCLC} \\
\midrule
 train& 71.3 & 69.1 & 51.4 & 86.4 & 87.4 & 87.4 & 86.3 & 85.3 & 87.3 & {\bf 89.3 } & {\bf 89.3 } \\
 & \ (4.2) & \ (2.8) & \ (3.4) & \ (1.2) & \ (2.7) & \ (2.7) & \ (1.2) & \ (0.8) & \ (1.7) & \ (0.6) & \ (0.6) \\
 \hline
 test&  29.6 & 27.9 & {\bf 46.0 } & 12.7 & 15.7 & 15.7 & 12.8 &  \ \ 9.9 & 14.8 & 24.9 & 24.9 \\
 & \ (3.3) & \ (8.5) & \ (2.1) & \ (2.7) & \ (4.2) & \ (4.2) & \ (2.6) & \ (0.2) & \ (3.8) & \ (1.9) & \ (1.9) \\
 \bottomrule
\end{tabular}
\end{table}

\begin{table}
\centering
	\caption{Mean accuracy (over $5$ runs) and standard deviation at test time for (n) $2,6,10$ environments, (d) across datasets, and (c) for the two color correlations (b01,b11).}
	\label{tab:environments}
%sorry i am done	
% 	\small
\footnotesize
\begin{tabular}{rr|lllllllllll}
\toprule
& & \rot{C-BVIRM} &  \rot{C-VIRMG} &  \rot{C-VIRMv1} &  \rot{ERM} &  \rot{EWC} &  \rot{GEM} &  \rot{IRMG} &  \rot{IRMv1} &  \rot{MER} &  \rot{VCL} &  \rot{VCLC} \\
\midrule
n & 2 & 29.6 & 27.9 & {\bf 46.0 } & 12.7 & 15.7 & 15.7 & 12.8 &  9.9 & 14.8 & 24.9 & 24.9 \\
 &  & (3.3) & (8.5) & (2.1) & (2.7) & (4.2) & (4.2) & (2.6) & (0.2) & (3.8) & (1.9) & (1.9) \\
 & 6 & 28.8 & 27.5 & {\bf 47.1} & 11.1 & 15.6 & 15.6 & 12.2 &  9.7 & 15.3 & 15.4 & 15.4 \\
 &  & (4.1) & (2.6) & (3.0) & (2.9) & (5.7) & (5.7) & (2.8) & (0.3) & (4.4) & (1.1) & (1.1) \\
 & 10 & 21.8 & 25.2 & {\bf 31.0} & 10.2 & 17.7 & 17.7 & 12.4 & 10.2 & 15.5 & 10.8 & 10.8 \\
 &  & (2.4) & (4.5) & (7.1) & (0.2) & (5.5) & (5.5) & (2.1) & (0.2) & (4.1) & (0.2) & (0.2) \\
d & MNIST & 29.6 & 27.1 & {\bf 46.8 } & 12.7 & 15.7 & 15.7 & 12.8 &  9.9 & 14.8 & 24.9 & 24.9 \\
 &  & (3.3) & (7.6) & (2.6) & (2.7) & (4.2) & (4.2) & (2.6) & (0.2) & (3.8) & (1.9) & (1.9) \\
 & Fa-MNIST & {36.3 } & 26.7 & {\bf 48.2 } & 10.7 & 15.4 & 15.3 & 10.8 &  9.9 & 13.2 & 24.9 & 24.9 \\
 &  & (4.3) & (8.7) & (3.6) & (1.5) & (5.1) & (5.4) & (1.4) & (0.2) & (2.8) & (2.0) & (2.0) \\
 & KMNIST & { 32.8 } & 24.2 & {\bf 46.5 } & 12.0 & 14.0 & 14.0 & 12.1 &  9.9 & 15.6 & 24.9 & 24.9 \\
 &  & (4.6) & (6.0) & (1.9) & (2.1) & (3.5) & (3.5) & (2.4) & (0.2) & (4.1) & (2.0) & (2.0) \\
 & EMNIST & { 32.1 } & 25.0 & {\bf 45.9 } & 10.8 & 15.3 & 14.8 & 10.8 & 10.0 & 12.6 & 24.9 & 24.9 \\
 &  & (4.6) & (7.5) & (2.5) & (1.0) & (3.6) & (3.7) & (1.2) & (0.2) & (2.3) & (2.0) & (2.0) \\
c & b01 & 29.6 & 27.1 & {\bf 46.8 } & 14.9 & 18.8 & 18.8 & 14.6 &  9.8 & 18.0 & 24.9 & 24.9 \\
 &  & (3.3) & (7.6) & (2.6) & (0.8) & (1.3) & (1.3) & (0.8) & (0.1) & (1.0) & (2.0) & (2.0) \\
 & b11 & {38.8 } & 23.9 & {\bf 43.3 } &  9.9 & 12.5 & 12.5 &  9.8 &  9.9 & 11.5 & 24.9 & 24.9 \\
 &  & (4.2) & (6.9) & (4.5) & (0.2) & (3.6) & (3.6) & (0.1) & (0.2) & (2.0) & (2.0) & (2.0) \\
 \bottomrule
\end{tabular}
\end{table}

\subsection{Results}
% In this section we present the results. Tab.\ref{tab:synthetic} shows the performance at training and testing. The IRMG as proposed in \cite{ahuja2020invariant} in the sequential case is not able to discover an invariant classifier, while the Bayesian approach obtains a very low training accuracy but learn a invariant classifier.

Table~\ref{tab:MNIST2} lists the training and test accuracy on the MNIST dataset with the color correction b01 (see Figure~~\ref{fig:color_correlation} left). Since we introduced label noise by randomly flipping $25$ percent of the given labels, a hypothetical optimal classifier would be able to achieve an accuracy of $75\%$ in both training and test environments.
ERM, IRMv1, and IRMG perform poorly in the setup where environments are given sequentially. Similarly, reference continual learning methods also fail to learn invariant representation in the new environment. 
As these models are learning to mainly use spurious features for the classification problems at hand, here: the colors of the digits (red$\sim$odd; green$\sim$even), they perform poorly (much worse than a random baseline) when the spurious feature properties are inverted (green$\sim$odd; red$\sim$even).
In contrast, our variational extensions to both IRM and IRMG achieve a classification accuracy higher than $45\%$ on the test data. This implies that our model is not relying exclusively on spurious correlations present in the color of digits. By comparing the performance between C-VIRMv1 and C-BVIRM, we conclude that (1) our proposed bilevel invariant risk minimization framework (i.e., the BIRM in Definition 1) is an effective alternative to the original formulation~\cite{arjovsky2019invariant}; and (2) ADMM is effective in solving the BIRM optimization problem and has the potential to improve the generalization performance.
In addition, one can observe that the KL divergence term in VCL and our framework significantly improves the test accuracy with respect to the baseline counterparts. This result further justifies our motivation of using a variational Bayesian framework for the problem of continual invariant risk minimization. 
Table~\ref{tab:environments} lists the accuracy on the test environment for: (n) (upper rows) an increasing number of sequential environments (d) (central rows) different datasets, and (c) (lower rows) the two given color correlation schemes. We can observe that there is a general trend in the results. IRMG and IRM, with an accuracy of less than $10\%$, are not able to learn invariant models. Similarily, the continual learning reference methods (MER, EWC, MER, VCL, VCLC) also fail with a test accuracy of under $25\%$. The proposed methods on the other hand provide mechanism to learn more robust features and classification models. The higher variance of the accuracy is caused by the stochastic nature of the variational Bayesian formulation. 

%From the numberical simulation we draw the following observations : 1) ERM, IRMv1 and IRGM perform poorly in sequential environment setup; 2) standard CL methods fails to learn invariant representations in new environment; 3) by reformulating IRM and IRGM in a Bayesian variational framework, we significantly improve their generalization ability; 4) ADMM provides an effective alternative solution to solve Bilevel IRM optimization.

% \subsection{Integring EIIL into Continual invariant learning}
\subsection{Environment Inference for Continual invariant learning}
% \textcolor{blue}{In practical applications, the environmental labels are usually unavailable, which means that it is difficult or impossible to manually partitioning the training set into ``domains" or ``environments". 
% In order to generalize our continual invariant learning models to an environment-agnostic setting, we make use of the recently proposed Environment Inference for Invariant Learning (EIIL) by~\cite{creager2020environment} to automatically infer environment partitions from observational training data, and integrate EIIL into our continual invariant learning models.}

% \textcolor{blue}
{In practical applications, the environmental labels are usually unavailable, which means that it is difficult or impossible to manually partitioning the training set into ``domains" or ``environments".
% Moreover, when observational data is collected
In order to generalize our continual invariant learning models to an environment-agnostic setting, we leverage the recently proposed Environment Inference for Invariant Learning (EIIL) by~\cite{creager2020environment} to automatically infer environment partitions from observational training data, and integrate EIIL into our continual invariant learning models.}

% \textcolor{blue}
{We take our proposed C-VIRMv1 as an example. 
According to Table~\ref{tab:MNIST_EIIL_main}, it is easy to observe that inferring environments directly from observational data (using EIIL) has the potential to improve (continual) invariant learning relative to using the hand-crafted environments. Moreover, C-VIRMv1 with EIIL improves both training and test accuracy, compared with IRMv1 with EIIL.
In fact, this environment partition strategy also enables invariant learning with only one environmental data. 
Table~\ref{tab:MNIST_EIIL_sample2_main} further suggests that the generalization accuracy improves for both IRMv1 and C-VIRMv1 as the number of training samples increases. Again, we observed that, when combined with EIIL, C-VIRMv1 always outperforms IRMv1.}

% Table~\ref{tab:MNIST_EIIL_sample2_main} further suggests an improved generalization accuracy for both IRMv1 and C-VIRMv1 with the increase of number of training samples. Again, we observed that, when combined with EIIL, C-VIRMv1 always outperforms VIRMv1.}
% Table~\ref{tab:MNIST_EIIL_sample2_main} shows the performance of the environment inference for increasing number of samples. We notice that when the number of samples is low, it is difficult to learn an invariant model, but as the number of samples increase, so do the test accuracy of the learned invariant model. Thus, generalization accuracy improves for both IRMv1 and C-VIRMv1 with the increase of number of training samples. Again, we observed that, when combined with EIIL, C-VIRMv1 always outperforms VIRMv1.}

\begin{table}
\centering
	\caption{Mean accuracy (over $10$ runs) on train and test environments when training off-line on $2$ environments on Colored-MNIST, with the EIIL. ($pc_1=0.2,pc_2=0.1,50'000$ samples)}
	\label{tab:MNIST_EIIL_main}
%sorry i am done	
% 	\small
\footnotesize
{\color{black}
\begin{tabular}{rllll}
\toprule
&{IRMv1} & & {C-VIRMv1} &  \\
& Train & Test & Train & Test  \\
\midrule
 No EIIL & 70.73	(1.16) &	{\bf 67.48 }	(1.96) &	{\bf 70.99 }	(0.90) &	66.60	(2.66) \\
 \hline
 EIIL & { 73.78 }	(0.61) &	67.96	(3.01) &	{\bf 75.29 }	(0.53) & {\bf	68.40 }	(1.11)  \\
 \bottomrule
\end{tabular}
}
\end{table}

\begin{table}
	\caption{Mean accuracy (over $5$ runs) on train and test environments when training on $1$ environment on Colored-MNIST, with and without EIIL. ($pc_1=0.1$)}
	\label{tab:MNIST_EIIL_sample2_main}
\centering
\footnotesize
{\color{black}
\begin{tabular}{rllllllll}
\toprule
% ns\_train & \multicolumn{2}{l}{train\_acc\_sf} & \multicolumn{2}{l}{test\_acc\_sf} & \multicolumn{2}{l}{train\_acc\_mf} & \multicolumn{2}{l}{test\_acc\_mf} & \multicolumn{2}{l}{train\_acc\_st} & \multicolumn{2}{l}{test\_acc\_st} & \multicolumn{2}{l}{train\_acc\_mt} & \multicolumn{2}{l}{test\_acc\_mt} \\
        %  &         mean &  std &        mean &  std &         mean &  std &        mean &  std &         mean &  std &        mean &  std &         mean &  std &        mean &  std \\
% \midrule
&\multicolumn{4}{l}{Without Environment Inference} & \multicolumn{4}{l}{With Environment Inference (EIIL)}  \\
&\multicolumn{2}{l}{IRMv1}  & \multicolumn{2}{l}{C-VIRMv1} &\multicolumn{2}{l}{IRMv1}  & \multicolumn{2}{l}{C-VIRMv1}  \\
$N_s$ & Train & Test & Train & Test & Train & Test & Train & Test \\
\midrule
    1'000 &   93.7  (0.7) &       13.5 (1.7) &        94.1 (1.1) &       13.7 (1.5) &        95.5 (0.3) &       12.7 (2.0) &        96.0 (0.4) &       16.5 (6.0) \\
    2'000 &   91.5 (0.4) &       12.7 (0.9) &        91.1 (0.7) &       11.9 (0.9) &        92.6 (0.4) &       27.8 (2.8) &        93.3 (0.7) &       29.3 (3.4) \\
    5'000 &   90.2 (0.4) &       10.5 (1.1) &        90.1 (0.4) &       10.6 (0.7) &        91.6 (0.4) &       29.6 (4.8) &        91.6 (0.9) &       30.6 (3.2) \\
   10'000 &   89.9 (0.3) &       10.1 (0.5) &        90.0 (0.2) &       10.1 (0.1) &        85.3 (1.0) &       42.9 (3.9) &        83.7 (1.2) &       50.4 (2.3) \\
   20'000 &   90.0 (0.2) &       10.1 (0.2) &        90.1 (0.2) &       10.1 (0.0) &        77.2 (1.2) &       57.4 (2.2) &        77.9 (1.1) &       57.6 (2.0) \\
   50'000 &   90.1 (0.1) &        9.7 (0.4) &        90.0 (0.1) &       10.0 (0.4) &        73.9 (0.5) &       67.2 (1.2) &        74.0 (0.5) &       67.3 (1.0) \\
\bottomrule
\end{tabular}
}
\end{table}

\section{Conclusions}

We aim to broaden the applicability of IRM to settings where environments are observed sequentially. We show that reference approaches fail in this scenario. We introduce a variational Bayesian approach for the estimation of the invariant models and a solution based on ADMM. We evaluate the proposed approach with reference models, including those from continual learning, and show a significant improvement in generalization capabilities. %We also demonstrate that the use of KL divergence promote the discovery of models whose support is the intersection of the support of the models of each separate environment. 

% %\bibliographystyle{aaai}
% %\bibliographystyle{plain}
% %\bibliography{iclr2021_conference}

\bibliographystyle{iclr2021_conference}
\bibliography{references}

\clearpage

\appendix
\section{Supplementary Material}

\subsection{Variational Invariant Risk Minimization games}

%We than propose a Variational IRM problem, where training environment are seen at the same time, we than extend to the case or continual learning.

% Let us denote $D^i$ the training data from the $i$-th environment $e^i$ and $\theta$ the network parameters. Here, $\theta$ refers to either $\phi$, $w$, or $w\circ\phi$. From a variational Bayesian perspective, a prior distribution $p(\theta)$ is placed over $\theta$ and our goal is to maximize the posterior distribution $p(\theta|D^{1:t})$. In most of the cases, the posterior distribution is intractable and is approximated with a variational distribution $q(\theta)$. If we define the approximate distribution as the prior, i.e., $q(\theta)=p(\theta)$, and quantify the closeness between $p(\theta|D^{1:t})$ and $q(\theta)$ with the KL divergence, we obtain the variational objective of IRMG as:
% \begin{eqnarray}
% \min_{q_\phi} && \E_{\phi \sim q(\phi)} \{\ell(y,\bar{w} \circ \phi)\} - \beta \text{KL} (q_\phi||p_\phi) \\
% \text{s.t.} && q_{w_e} = \arg \min_{q_{w_e}} \E_{w \sim q_{w_e}} \{ \ell (y,\frac1{|E|} (w+w_{-e}) \} - \beta \text{KL} (q_{w_e}||p_w) \forall e \in E^{\text{tr}}
% \end{eqnarray}

We now consider the IRMG objective and extend it with the variational Bayesian inference.
If we observe all environment at the same time, the prior of the single environment is data independent. From Equation~\ref{eq:bayeian_main}, we thus substitute $q_{t-1}(\theta)$ with a priors $p_\phi(\theta)$ and $q_w(\omega)$, where $\theta$ and $\omega$ are now the parameters of the two functions $\phi$ and $w$. While we substitute $q_t(\theta)$, with the variational distributions $q_\phi(\theta)$ and $q_w(\omega)$. 
The outer problem is now 
% \begin{eqnarray}
% \min_{q_\phi} && \E_{\phi \sim q(\phi)} \{\ell(y,\bar{w} \circ \phi)\} - \beta \text{KL} (q_\phi||p_\phi) \\
% \text{s.t.} && q_{w_e} = \arg \min_{q_{w_e}} \E_{w \sim q_{w_e}} \{ \ell (y,\frac1{|E|} (w+w_{-e}) \} - \beta \text{KL} (q_{w_e}||p_w) \forall e \in E^{\text{tr}}
% \end{eqnarray}
\begin{subequations} \label{eq:VIRMG}
\begin{eqnarray}
\min_{q_\phi} && \E_{\phi \sim q(\phi)} R^e (\bar{w} \circ \phi) + \beta \KL (q_\phi||p_\phi) \\
\text{s.t.} && q_{w_e} = \arg \min_{q_{w_e}} \E_{w \sim q_{w_e}} R^e(\frac1{|E|} (w+w_{-e}) \circ \phi ) + \beta \KL (q_{w_e}||p_w) \forall e \in E^{\text{tr}}
\end{eqnarray}
\end{subequations}

where
$\bar{w} = \frac1{|E|} \sum_{e \in E^{\text{tr}}} w_e$, $ w_e \sim q_{w_e}(w)$ is the average classifier and
$w_{-e} = \sum_{e' \in E^{\text{tr}}, e' \ne e} w_{e'},  w_{e'} \sim q_{w_{e'}}(w) $
is the complement classifier. In the reformulation of the IRMG model, we weight the distance of the varional distribution to the prior with $\beta$. We notice how the difference of the variational formulation of IRMG differs on the presence of the mean on the distribution of the function over the variational distributions and the KL term. 

We can now finally extend IRMG when the environments are observed sequentially.  
Combining the definition of IRMG Eqs.~(\ref{eq:IRMG}) with the continual bayesian learning Equation~(\ref{eq:bayeian_main}), we obtain the variational objective of IRMG in sequential environment case.
\begin{subequations} \label{eq:CVIRMG}
\begin{eqnarray}
\min_{q_\phi} && \E_{\phi \sim q(\phi)} \{\ell(y,\bar{w} \circ \phi)\} + \beta \KL (q_\phi||q^{t-1}_\phi) \\
\text{s.t.} &&\bar{w} = \frac1{2} (w + w_{t-1}), w \sim q_{w}(w), w_{t-1} \sim  q^{t-1}_{w}(w) \\
&& q_{w} = \arg \min_{q_{w}} \E_{w \sim q_{w_e},\phi \sim q_\phi} \{ \ell (y,\frac1{2} (w+w_{t-1}) \circ \phi \}  + \beta \KL (q_{w_e}||q^{t-1}_w)
\end{eqnarray}
\end{subequations}
We can similarly extend the definition of IRMv1 when all environments are seen at the same time and sequentially. 
% we can alternatively $$\bar{w} = \frac1{t} w +  \frac{t-1}{t} w_{t-1}$

% \begin{comment}
% \subsection{Alternative derivation of Invariant Risk Minimization}
% In this section we reformulate the IRM principle that will allow us to propose an algorithm to solve the learning of causal models.

% \subsubsection{IRM definition}
% We restate the IRM model from \cite{arjovsky2019invariant}
% \begin{definition} [IRM]
% Give a set of mapping $H_\phi$ and classifier $H_w$, an {\bf invariant predictor } on a set of environments $E$ is said to satisfy the Invariant Risk Minimization (IRM) if it is the solution of the following problem
% \begin{eqnarray}
% 	\min_{\phi \in H_\phi,w \in H_w} && \sum_{e \in E} R^e(w \circ \phi) \\
% 	\text{s.t.} && w \in \arg \min_{w_e \in H_w} R^e(w_e \circ \phi), \forall e \in E
% \end{eqnarray}
% \end{definition}
% Based on the previous definition, a more practical version of the problem is defined \cite{arjovsky2019invariant},
% \begin{definition} [IRMv1]
% 	Similar to the condition of IRM, IRMv1 is the solution of
% 	\begin{eqnarray}
% 		\min_{\phi \in H_\phi} && \sum_{e \in E} R^e(\phi) + \lambda ||\nabla_{w|w=1.0} R^e(w \phi)||^2 , \forall e \in E,
% 	\end{eqnarray}
% where $w$ is a scalar evaluated in $1.$.
% \end{definition}
% This second condition thus learn a unique invariant model $\phi$, possibly with multiple output. We notice that this invariant model would be different in case of multi task learning, were if task do not share the same model would not be able to predict multi tasks.
% \end{comment}

\subsection{Mean Field Parametrization and reparametrization trick}
When we want to implement Equation~\ref{eq:bvirm_admm_main} and Equation~\ref{eq:bvirm_admm_lagrangian} and the different variation, we use the mean field approximation and the reparametrization trick \cite{kingma2013auto}. In this case the density function of our model is parameterized by $\theta$ and $\omega$ and constraints becomes
$ \nabla_{q(\omega)} Q_w^e(\omega^+_e , \theta) = 0 \to  \nabla_{\omega} Q_w^e(\omega , \theta) = 0 $. If we then parametrize $\mu(\omega_\mu)$ and $\sigma(\omega_\sigma)$ the mean and standard deviation and model the distribution as $q_\omega(w) = \mu(\omega_\mu) + \epsilon \sigma(\omega_\sigma) $, with $\epsilon \sim N(0,1)$
We now want to compute the gradient (in the following we ignore the dependence on the $\phi$ and its parameters)
$$
\nabla_\omega Q(\omega) = \nabla_\omega \E_{w \sim q(\omega)} R(w \circ \phi) + \beta \nabla_\omega \KL{(q(\omega)||p) }
$$
The second term is
$$
\nabla_\omega \KL{(q||p) }  = \nabla_\mu \KL{(q||p) } \nabla_\omega \mu  + \nabla_\sigma \KL{(q||p) } \nabla_\omega \sigma
$$
with
$$
\nabla_\omega \mu = 1 ,  \nabla_\omega \sigma = \frac1{\epsilon}
$$
$$
\nabla_\mu \KL{(q||p) }  = - \sigma_p^{-1} (\mu_p-\mu_q)
$$
$$
\nabla_\sigma \KL{(q||p) }  = - \diag{ \sigma_q}^{-1} + \diag {\sigma_p}^{-1}
$$
where we assume $\sigma_p,\sigma_q$ to be diagonal, in this way the previous equation can be evaluated element-wise and where the $\KL{(q||p)}$ is defined as
$$
\KL{(q||p) }  = \ln \frac{|\Sigma_p|}{|\Sigma_q|} -n +\tr{\Sigma_p^{-1} \Sigma_q} + (\mu_p-\mu_q)^T  \Sigma_p^{-1} (\mu_p-\mu_q)
$$
The first term is evaluated by Monte Carlo sampling
$$
\nabla_\omega \E_{w \sim q(\omega)} R(w)  \approx  \nabla_\omega  \frac1{N} \sum_{i=1}^N R(w_i)
$$
with
$$
w_i = \mu(\omega) + \epsilon_i \odot \sigma(\omega)
$$
and $w_i \sim N(0,1)$. Also in this case
$$
\nabla_\omega  \frac1{N} \sum_{i=1}^N R(w_i)  = \nabla_\mu \frac1{N} \sum_{i=1}^N R(w_i) \nabla_\omega \mu  + \nabla_\sigma \frac1{N} \sum_{i=1}^N R(w_i) \nabla_\omega \sigma
$$

\subsection{The BIRM-ADMM Algorithm}
We observe that to solve BIRM we can use Lemma \ref{th:gadmm} 
% (supplementary material) 
and write the following algorithm
\begin{subequations} \label{eq:admm1}
\begin{eqnarray} %\label{eq:admm1}
	w_e^+ &=& \arg \min_{w_e} L_\rho(w_e , u_e^-,w^-,v_e^-) , \forall e \in E\\
	w^+ &=& 1/|E| \sum_i (w_e+u_e)\\
	u_e^+ &=& u_e^- + (w_e^+ - w^+) \\
	v_e^+&=&v_e^- +  \nabla_{w} R^e(w^+_e \circ \phi)
\end{eqnarray}	
\end{subequations} 
where
% \begin{eqnarray} \label{eq:admm2}
% 	L_\rho(w_e,u_e,w,v_e) &=& R^e (w_e \circ \phi) \nonumber  \\
% 	&& + \frac{\rho_0}{2}  \|w_e - w +u_e \|^2 \nonumber  \\
% 	&&  +  \frac{\rho_1}{2}  \|  \nabla_{w} R^e(w_e \circ \phi) + v_e\|^2
% \end{eqnarray}
\begin{eqnarray} \label{eq:admm2}
	L_\rho(w_e,u_e,w,v_e) &=& R^e (w_e \circ \phi)  + \frac{\rho_0}{2}  \|w_e - w +u_e \|^2 +  \frac{\rho_1}{2}  \|  \nabla_{w} R^e(w_e \circ \phi) + v_e\|^2
\end{eqnarray}
We denote $.^+,.^-$ the values of the variable after and before the update. In order to implement the method we use the SGD to update the model $w_e$ and in a outer loop updating for $\phi$.

%The gradient can be approximated with the difference between updates, but it would be normalized by the norm of the change in the parameters:
%$$
%\nabla_{w} R^e(w_e \circ \phi)  \approx (R^e(w_e \circ \phi) - R^e(w^-_e \circ \phi)) / || w_e - w^-_e ||^2
%$$
%but would not be accurate when the model converges.

\begin{algorithm}
	\label{alg:IRM-ADMM}
	\SetAlgoLined
	\KwResult{$w \circ \phi$ : feature extraction and classifier for the environment $E$ }
	\tcp{Randomly initialize the variables}
	$w, w_e, u_e, v_e, \phi \gets \text{Init()} $ \;
%	$\phi \gets \text{Init()} $ \;
	\tcp{Outer (on $\phi$) and Inner loop (on $w$) }
	\While{not converged}{
		\tcp{ Update $\phi$ using stochastic gradient descent(SGD) }
		$\phi =  \text{SGD}_{\phi} (\sum_e R^e (w \circ \phi ))  $ \;
		\For {$k = 1, \dots, K$}{
			\For {$e \in E$}{
				$w_e = \text{SGD}_{w_e} L_\rho(w_e , u_e, w, v_e) $ \;
				$w = 1/|E| \sum_e (w_e+u_e)$ \;
				$u_e = u_e + (w_e - w)$ \;
				$v_e = v_e +  \nabla_{w} R^e(w_e \circ \phi)$ \;
			}
		}
		
	}
	\caption{$w,\phi \gets $ BIRM-ADMM($E,R^e$) ADMM version of the Bilevel IRM Algorithm}
\end{algorithm}

\subsection{Variational Invariant Risk Minimization}
\begin{definition} [VIRM] \label{def:VIRM}
	Give a set of distribution over the mapping $P_\phi$ and a distribution over the set of classifier $P_w$, a {\bf variational invariant predictor } on a set of environments $E$ is said to satisfy the Variational Invariant Risk Minimization (VIRM) if it is the solution of the following problem
%	\begin{eqnarray}
%		\min_{q_\phi \in P_\phi} && \sum_{e \in E} \E_{w \sim q_w, \phi \sim q_\phi} Q_\phi^e(w \circ \phi) \\
%		\text{s.t.} && q_w \in \arg \min_{q^e_w \in P_w}  \E_{w \sim q_w, \phi \sim q_\phi} Q_w^e(w_e \circ \phi), \forall e \in E \\
%		\text{where} && Q_\phi^e(w \circ \phi) = R^e(w \circ \phi) - \beta KL(q_\phi||p_\phi), \\
%		&& Q_w^e(w \circ \phi) = R^e(w \circ \phi) - \beta KL(q_{w}||p_{w})
%	\end{eqnarray}
    \begin{subequations} \label{eq:VIRM}
	\begin{eqnarray}
	 \mathop{ \min_{q_\phi \in P_\phi}}_{q_w \in P_w}  && \sum_{e \in E}  Q_\phi^e(q_w,q_\phi)\\
	\text{s.t.} && q_w \in \arg \min_{q^e_w \in P_w}  Q_w^e(q_w,q_\phi), \forall e \in E \\
	\text{where} && Q_\phi^e(q_w,q_\phi) =  \mathop{\E_{w \sim q_w}}_{\phi \sim q_\phi} R^e(w \circ \phi) + \beta \KL(q_\phi||p_\phi) {\color{black} + \beta \KL(q_{w}||p_{w})}, \\
	&& Q_w^e(q_w,q_\phi) =  \mathop{\E_{w \sim q_w}}_{\phi \sim q_\phi}  R^e(w \circ \phi) + \beta \KL(q_{w}||p_{w})
\end{eqnarray}
\end{subequations}
	and $p_{\phi},p_{w}$ are the priors of the two distributions.
\end{definition}

\subsection{Bilevel Alternative Formulation}
We state here a general result on solving Bilevel Optimization Problems
\begin{lemma}[Bilevel Reformulation] \label{th:reform}
\begin{subequations} \label{eq:BI1}
\begin{eqnarray}
	\min_{x} && F(x,y) | G(x,y(x)) \le 0 \\
	\text{s.t.} && y(x) \in \arg \min_y f(x,y) | g(x,y) \le 0
\end{eqnarray}
\end{subequations} 
then we can solve the equivalent problem
\begin{subequations} \label{eq:BI2}
\begin{eqnarray}
	\min_{x,y,u} && F(x,y) | G(x,y(x)) \le 0, \\
	&&  \nabla_y L(x,y,u)=0, \\
	&& u \ge 0, \\
	&& g(x,y) \le 0, \\
	&& u^Tg(x,y)=0 \\
	L(x,y,u) &=& f(x,y) + u^Tg(x,y)
\end{eqnarray}
\end{subequations} 
\end{lemma}
{\color{black} 
\begin{proof} [Proof of Lemma~\ref{th:reform} ]
Lemma~\ref{th:reform} follows by applying the Karush-Kuhn-Tucker conditions (Chapter 5 \cite{boyd2004convex}) to Eq.\ref{eq:BI1}, where the Lagrangian function is $L(x,y,u) = f(x,y) + u^Tg(x,y)$.
\end{proof}
\begin{lemma}[Equivalence of Definition \ref{def:BIRM}] \label{th:equivalent1}
Definition \ref{def:BIRM} is equivalent to Eq. \ref{eq:IRM}, the Invariant Risk Minimization.
\end{lemma}
\begin{proof} [Proof of Lemma~\ref{th:equivalent1} ]
The result follows by apply Lemma~\ref{th:reform} to Eq.\ref{eq:IRM}.
\end{proof}
\begin{lemma}[Definition \ref{def:BVIRM}] \label{th:equivalent2}
Definition \ref{def:BVIRM} is the extension of Eq. \ref{eq:BIRM}, the Bilevel Invariant Risk Minimization, when the function is described by the distributions of their variable $\phi$ and $w$.
\end{lemma}
\begin{proof} [Proof of Lemma~\ref{th:equivalent2} ]
The result follows by inspecting Eq. \ref{eq:BIRM}. The equation requires the minimisation of the aggregated loss function, which is now, from Eq.\ref{eq:bayeian_main}: 
% $$
% Q = E_{(x,y) \sim D_t} \E_{\theta \sim q(\theta)} \{\ell(y,f_\theta (x))\} +
% \KL{\left(q (\theta) || q_{t-1}(\theta)\right)}, 
% $$
\begin{equation} \label{eq:obj_bvirm}
Q_\phi^e(q_w,q_\phi) =  \mathop{\E_{w \sim q_w}}_{\phi \sim q_\phi} R^e(w \circ \phi) + \beta \KL(q_\phi||p_\phi)+ \beta \KL(q_{w}||p_{w}),
\end{equation}
where we have separated the two contributions in $\phi$ and $w$, and used genetic prior distributions $p_\phi$ and $p_w$. This is by the additive property of KL divergence:
\begin{equation} \label{eq:kl_add}
\KL(q_\phi q_{w}||p_\phi p_{w}) = \KL(q_\phi||p_\phi) + \KL(q_{w}||p_{w}),
\end{equation}
since we model the two distributions independently, i.e. $q_{\phi,w} = q_\phi q_{w}$ and $q_{\phi,w} = p_\phi p_{w}$. 
Since the classifiers' losses shall be minimal for all environments, this condition is substituted by requiring the gradient with respect to $q_w$ to be zero, $\forall e$. The gradient w.r.t. $q_w$ of the second term of Eq.\ref{eq:obj_bvirm} is zero. 
\end{proof}
}

{\color{black}
\subsection{Theorem 3 and IRM connection}
\subsubsection{Sequential Information Projection}
In Theorem.\ref{th:i-proj}, we show that the Information Projection (IP) shrinks the support of the output distribution. 
\begin{lemma} \label{th:sequential_i-proj}
If we have a sequence of families of distributions $P_i$. Let $p_1 \in P_1$ and 
$$
p_{i+1} = \arg \min_{p \in P_{i+1}} \KL(p,p_{i})
$$ 
then 
$$
\supp p_{i} \subseteq \bigcap_{j\le i} \supp (P_j)
$$
\end{lemma}
\begin{proof}[Proof of Lemma~\ref{th:sequential_i-proj}]
We have that $\forall i, \supp p_i \subseteq  \supp (P_i)  \bigcap  \supp (P_{i-1}) $, where the first condition follows from $p_i \in P_i$ in the minimization and the second from Theorem~\ref{th:i-proj}. The results follows by iterating the property.
\end{proof}

\subsubsection{IRM and Information Projection}
\begin{wrapfigure}[10]{R}{0.4\textwidth}
% \vspace{-4mm}
\centering
\includegraphics[width=0.3\textwidth]{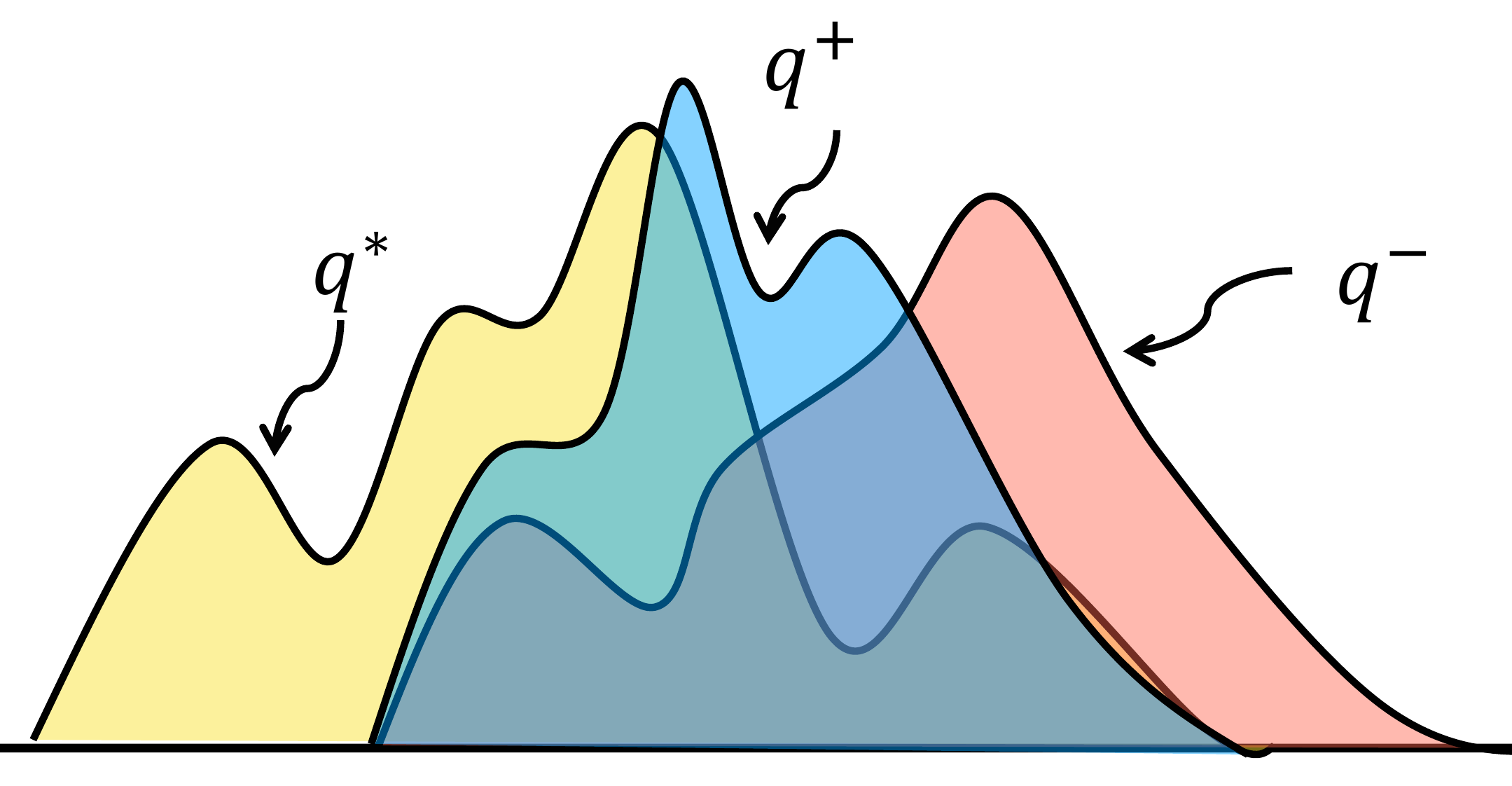}
\caption{\label{fig:irm_invariant_distribution} Sequential IRM projection of distributions,  where  $q^+=\arg \min \KL(p||q^-) + \KL(p||q^*)$ } 
\end{wrapfigure}
We show now two ways to state the connection of the IRM principle and the sequential IP.
Let $q^-$ be the distribution of the previous environment and $R(q)$ the loss function of the current environment, where $q$ denotes the distribution of the network parameters.
Let 
$$
q^* = \arg \min_q R(q)
$$ 
be the optimal distribution for the current environment. We can then consider the Taylor expansion of the parameters distribution around the optimal distribution as
$$
R(q) = R(q^*)+\Delta q^T \nabla_q R(q^*)
$$
we can compute the new distribution as
\begin{eqnarray*}
\Delta q^* &=& \arg \min_{\Delta q} \KL( q^*+\Delta q||q^-) \\
\text{s.t.} && \Delta q^T \nabla_q R(q^*) \le \epsilon
\end{eqnarray*}
and then 
$$
p^+ = q^*+\Delta q^*
$$
Or alternatively 
\begin{eqnarray*}
p^*,q^* &=& \arg \min_{p,q} \KL( p||q^-) + \KL( p||q) \\
\text{s.t.} && \nabla_q R(q) = 0
\end{eqnarray*}
and then
$$
q^+ = p^*
$$
Or more simply
\begin{eqnarray} \label{eq:irm_ip}
q^+ &=& \arg \min_{p} \KL( p||q^-) + \KL( p||q^*) 
\end{eqnarray}
This last equation, shows how the new distribution is the intersection of the optimal distribution at the previous step $q^-$ and the current optimal distribution $q^*$. Fig~\ref{th:sequential_i-proj} shows visually, how the new distribution is the result of projecting into two distributions $q^*$ and $q^-$.
% , or equivalently the product distribution $q^*$ and $q^-$.

% IRM-IP.pdf
}

{\color{black}
\subsection{Out of distribution Generalization}
The question arises if the property of generalization to out of distributions given by Theorem 9 in \cite{arjovsky2019invariant} also holds for BIRM and BVIRM. 

\begin{lemma} \label{th:invariant1}
If $\phi$ and $w$ are linear functions and $w \circ \phi = \Phi ^T w$ is a solution of Eq.\ref{eq:BIRM} it then satisfies 
$$
\Phi \E_{X^e} \left[ X^e X^{eT} \right] \Phi^T w = \Phi \E_{X^e,Y^e} \left[ X^e Y^{eT} \right]
$$
\end{lemma}
\begin{proof}
Lemma \ref{th:invariant1} follows from the fact that 
\begin{eqnarray}
\nabla_{w} R^e(w \circ \phi) &=& \Phi \E_{X^e} \left[ X^e X^{eT} \right] \Phi^T w - \Phi \E_{X^e,Y^e} \left[ X^e Y^{eT} \right] \\
&=& 0 
\end{eqnarray}
\end{proof}
% Corollary
For BIRM thus Theorem 9 of \cite{arjovsky2019invariant} applies directly. 
A similar results holds for the the BVIRM model
\begin{lemma} \label{th:invariant2}
If $\phi \sim p_\phi$ and $w \sim p_\phi $ are linear functions and $w \circ \phi = \Phi ^T w$ is a solution of Eq.\ref{eq:BVIRM}, with $\beta=0$, it then satisfies 
$$
\E_{\phi \sim q_\phi} \left\{ {\Phi} \E_{X^e} \left[ X^e X^{eT} \right] {\Phi}^T \right\} \bar{w} = \bar{\Phi} \E_{X^e,Y^e} \left[ X^e Y^{eT} \right]
$$
% $$
% \mathop{\E_{w \sim q_w}}_{\phi \sim q_\phi} 
% \bar{\Phi} \E_{X^e} \left[ X^e X^{eT} \right] \bar{\Phi}^T \bar{w} = \bar{\Phi} \E_{X^e,Y^e} \left[ X^e Y^{eT} \right]
% $$
where $\bar{\Phi} = \E_{\Phi \sim p_\phi} [\Phi]$ and $\bar{w} = \E_{w \sim p_w}  [w]$ are the mean values. 
\end{lemma}

% \begin{subequations}
% \label{eq:BVIRM}
% \noindent\begin{minipage}[t]{.5\linewidth}
% \begin{equation}
%  \mathop{ \min_{q_\phi \in P_\phi}}_{q_w \in P_w} \ \ \sum_{e \in E} Q_\phi^e(q_w,q_\phi)
% \end{equation}
% \end{minipage}%
% \begin{minipage}[t]{.5\linewidth}
% \begin{equation}
% \text{s.t.} \ \ \nabla_{q_w}  Q_w^e(q_w,q_\phi) = 0, \forall e \in E,
% \end{equation}
% \end{minipage}
% \end{subequations}

% \begin{subequations} \label{eq:BVIRM_q}
% 	\begin{eqnarray}
% 	\text{with } && Q_\phi^e(q_w,q_\phi) =  \mathop{\E_{w \sim q_w}}_{\phi \sim q_\phi} R^e(w \circ \phi) + \beta \KL(q_\phi||p_\phi) {\color{black} + \beta \KL(q_{w}||p_{w})}, \\
% 	\text{and} && Q_w^e(q_w,q_\phi) =  \mathop{\E_{w \sim q_w}}_{\phi \sim q_\phi}  R^e(w \circ \phi) + \beta \KL(q_{w}||p_{w}),
% \end{eqnarray}
% \end{subequations}

\begin{proof}
Lemma \ref{th:invariant2} follows from the fact that 
\begin{eqnarray}
\nabla_{q_w}  Q_w^e(q_w,q_\phi) |_{\beta=0} &=&\nabla_{q_w}  \mathop{\E_{w \sim q_w}}_{\phi \sim q_\phi}  R^e(w \circ \phi) \\
&=& 0 
\end{eqnarray}
% that means for each $w \sim q_w$ then $w \circ \phi$ is minimal $\min_w R^e(w \circ \phi)$ or $ \nabla_{w} R^e(w \circ \phi) = 0$ thus
% \begin{eqnarray}
% \mathop{\E_{w \sim q_w}}_{\phi \sim q_\phi} \nabla_{w} R^e(w \circ \phi) &=& \bar{\Phi} \E_{X^e} \left[ X^e X^{eT} \right] \bar{\Phi}^T \bar{w} - \bar{\Phi} \E_{X^e,Y^e} \left[ X^e Y^{eT} \right] \\
% &=& 0 
% \end{eqnarray}
We now take the Fr\'echet directional derivative in the $\eta$ direction that is the limit of 
$$ 
\delta_{q_w,\eta} \mathop{\E_{w \sim q_w}}_{\phi \sim q_\phi}  R^e(w \circ \phi) = \lim_{\epsilon \to 0} \frac1{\epsilon} ( \mathop{\E_{w \sim q_w}}_{\phi \sim q_\phi}  R^e(((w+\epsilon \eta ) \circ \phi) -  \mathop{\E_{w \sim q_w}}_{\phi \sim q_\phi}  R^e(w \circ \phi) )
$$
which is obtained when we differentiate the distribution $q_w \to q_w +\epsilon \eta$. 
Since $\delta_{q_w,\eta} \mathop{\E_{w \sim q_w}}_{\phi \sim q_\phi}  R^e(w \circ \phi) = \mathop{\E_{w \sim q_w}}_{\phi \sim q_\phi}  2 \eta^T {\Phi} \E_{X^e} \left[ X^e X^{eT} \right] {\Phi}^T {w} -2  \eta^T {\Phi} \E_{X^e,Y^e} \left[ X^e Y^{eT} \right]  $ 
we can factorize for the direction $\eta$ and obtain
$$
\delta_{q_w} R^e(w \circ \phi) = 2 \mathop{\E_{w \sim q_w}}_{\phi \sim q_\phi}  {\Phi} \E_{X^e} \left[ X^e X^{eT} \right] {\Phi}^T {w} -  {\Phi} \E_{X^e,Y^e} \left[ X^e Y^{eT} \right] 
$$
We can now derive the Lemma by requiring $ \delta_{q_w} R^e(w \circ \phi) = 0$
\end{proof}
Theorem 9 of \cite{arjovsky2019invariant} now holds when $\phi$ has $\rank r>0$ in expectation with respect to the invariant distribution $q_\phi$, i.e. $\E_{\phi \sim q_\phi} \rank(\Phi) = r$.
% \begin{assumption}
% $$
% \mathop{\E_{w \sim q_w}}_{\phi \sim q_\phi}  {\Phi} \E_{X^e} \left[ X^e X^{eT} \right] {\Phi}^T {w} -  {\Phi} \E_{X^e,Y^e} \left[ X^e Y^{eT} \right] 
% $$
% \end{assumption}

% We show this with the following Lemma.
% \begin{lemma} \label{th:outofdistribution1}
% If $\phi$ and $w$ are linear functions and $w \circ \phi = \phi ^T w$ is an invariant predictor then Eq.\ref{eq:BIRMb} holds. 
% \end{lemma}
% \begin{proof}
% Lemma \ref{th:outofdistribution1} is true because Eq.
% \end{proof}

}

\subsection{Generalized ADMM}
The following generalization of ADMM holds:
\begin{lemma}[GADMM] \label{th:gadmm}
	Suppose we want to minimized
	\begin{eqnarray}
		\min_{x} && \sum_i f_i(x) | g_i(x) = 0, \forall i \in I
	\end{eqnarray}
	we can equivalently solve the following problem
	\begin{eqnarray}
		\min_{x_i,z} && \sum_i f_i(x_i) | x_i=z, g_i(x_i) = 0, \forall i \in I
	\end{eqnarray}
	using the following update role (scaled ADMM)
	
	\begin{subequations} \label{eq:gadmm_update}
	\begin{eqnarray}
		x_i^+ &=& \arg \min_{x_i} L_\rho(x_i,x_{-i}^-,u_i^-,z^-,v_i^-)  , \forall i \in I \\
		z^+ &=& 1/N \sum_i (x_i+u_i)\\
		u_i^+ &=& u_i^- + (x_i^+ - z^+) \\
		v_i^+&=&v_i^- + g_i(x_i^+)
	\end{eqnarray}	
	\end{subequations} 
	where the augmented Lagrangian
	\begin{eqnarray}
		L_\rho(x_i,u_i,z,v_i) &=& \sum_i f_i(x_i) + \frac{\rho_0}{2}  \sum_i \|x_i - z +u_i \|^2  +  \frac{\rho_1}{2}  \sum_i  \| g_i(x_i) + v_i\|^2
	\end{eqnarray}	
	and  $x_{-i} = \{x_j, j \ne i\}$ is the set of all other variable, expect the $i$-th.
\end{lemma}	

\subsection{Continual Variational Inference}

Following \cite{nguyen2018variational} we can state the following lemma.
\begin{lemma}[Variational Continual Learning] \label{th:vcl}
Suppose we have a sequence of datasets $D_i, i=1,\dots,t$ drown i.i.d, then the variational estimation of the distribution $q_t$ at step $t$ is given as projection on KL divergence 
$$
q_t (\theta) = \arg \min_{q(\theta)} \KL{\left(q (\theta) || \frac{1}{Z_t} q_{t-1}(\theta) p(D_t|\theta)\right)}
$$
with $Z_t=\int q_{t-1}(\theta) p(D_t|\theta) d\theta$ the normalization factor, which does not depends on $q$.
\end{lemma}
\begin{proof} [Proof of Lemma \ref{th:vcl}]
Let denote $D_1^t = \bigcup_{i=1}^t D_i$, from i.i.d. $p(D_1^t) = \prod_{i=1}^t p(D_i)$. We are interested to maximase the a posteriori probability of the paramters give the data 
$p(\theta|D_1^t)$
$$
p(\theta|D_1^t) = \frac1{p(D_t)} p(\theta|D_1^{t-1}) p(D_t|\theta)
$$
since 
\begin{eqnarray*}
p(\theta,D_1^t) &=& p(\theta|D_1^t)p(D_1^t) \\
&= & p(\theta)p(D_1^t|\theta) \\
&= & p(\theta)\prod_{i=1}^t p(D_i|\theta) \\
&= & p(\theta)p(D_1^{t-1}|\theta)p(D_t|\theta) \\
&= & p(\theta,D_1^{t-1})p(D_t|\theta) \\
&= & p(\theta|D_1^{t-1})p(D_1^{t-1})p(D_t|\theta)
\end{eqnarray*}
thus
\begin{eqnarray*}
p(\theta|D_1^t) &=& \frac1{p(D_1^t)} p(\theta|D_1^{t-1})p(D_1^{t-1})p(D_t|\theta) \\
&=& \frac{p(D_1^{t-1})}{p(D_1^t)} p(\theta|D_1^{t-1})p(D_t|\theta) \\
&=& \frac1{p(D_t)} p(\theta|D_1^{t-1})p(D_t|\theta) \\
\end{eqnarray*}
We now use a probability distribution which approximates the distribution at step $t-1$
$$
q_{t-1}(\theta) \approx p(\theta|D_1^{t-1})
$$
when then want to approximate at time $t$
$$
q_{t}(\theta) \approx  \frac1{p(D_t)} q_{t-1}(\theta) p(D_t|\theta)
$$
This can be obtain by minimizing the KL divergence of the variational distribution $q_t$ and the distribution induced by the previous step approximation, thus
$$
q_t (\theta) = \arg \min_{q(\theta)} \KL{\left(q (\theta) || \frac{1}{Z_t} q_{t-1}(\theta) p(D_t|\theta)\right)}
$$
\end{proof}

\begin{lemma}[VCLv2]\label{th:VCLv2}
The minimization of the VCL defined in Lemma~\ref{th:vcl}, is equivalent to solve the following minimization
$$
q_t (\theta) = \arg \max_{q(\theta)} 
\E_{\theta \sim q(\theta)} \{\log p(D_t|\theta)\} -
\KL{\left(q (\theta) || q_{t-1}(\theta)\right)}
$$
with $N_t$ i.i.d. samples 
\begin{eqnarray*}
\E_{\theta \sim q(\theta)} \{\log p(D_t|\theta)\} &=& \frac1{N_t} \sum_{i=1}^{N_t} E_{\theta \sim q(\theta)} \{\log p(y_i^t|\theta,x_t^t)\} 
\end{eqnarray*}
\end{lemma}
Where the second term can be computed in closed form for known distribution as for example with the Gaussian distributions, whereas the expectation can be approximated by Monte Carlo sampling. For a general loss function we can substitute the reconstruction probability with the loss function associated with a neural network parametrized by $\theta$ 
$$
\log p(y_i^t|\theta,x_t^t) \gets \ell(y_i^t,(w \circ \phi)_\theta (x_t^t)) 
$$ 
\begin{eqnarray*}
\E_{\theta \sim q(\theta)} \{\log p(D_t|\theta)\} 
&\gets&\frac1{N_t} \sum_{i=1}^{N_t} E_{\theta \sim q(\theta)} \ell(y_i^t,(w \circ \phi)_\theta (x_t^t)) \} 
\end{eqnarray*}
\begin{proof}[Proof of Lemma \ref{th:VCLv2}]
The Lamma follows from the definition of the KL diveregnce
\begin{eqnarray*}
\KL{\left(q (\theta) || \frac{1}{Z_t} q_{t-1}(\theta) p(D_t|\theta)\right)} &=& \E_q(\ln q(\theta) - \ln q_{t-1}(\theta)  - \ln p(D_t|\theta) + \ln Z_t) \\
&=& \E_q(\ln q(\theta) - \ln q_{t-1}(\theta)) - \E_q \ln p(D_t|\theta)  + \E_q \ln Z_t \\
&=& \KL ( q(\theta) || q_{t-1}(\theta)) - \E_q \ln p(D_t|\theta)  + \ln Z_t \\
\end{eqnarray*}
The last term does not depend on $q$. Thus the result follows. 
\end{proof}
If we substitute the log of the posterior probability with a specific loss function we obtain the following Corollary.  
\begin{corollary} [Continual Variational Bayesian Inference] \label{th:bayesian}
Given a loss function $\ell(y,\hat{y})$, the variational continual learning is formulated as 
% \begin{equation}\label{eq:14b}
%     q_t (\theta) = \arg \min_{q(\theta)} 
% \frac1{N_t} \sum_{i=1}^{N_t} \E_{\theta \sim q(\theta)} \{\ell(y_i^t,(w \circ \phi)_\theta (x_t^t))\} +
% \KL{\left(q (\theta) || q_{t-1}(\theta)\right)},
% \end{equation}
\begin{equation}\label{eq:bayeian}
    q_t (\theta) = \arg \min_{q(\theta)} 
\E_{(x,y) \sim D_t} \E_{\theta \sim q(\theta)} \{\ell(y,f_\theta (x))\} +
\KL{\left(q (\theta) || q_{t-1}(\theta)\right)},
\end{equation}
with $f_\theta = (w \circ \phi)_\theta$
\end{corollary}

\subsection{Proofs}
\begin{proof} [Proof of Theorem \ref{th:i-proj}]
Let first first recall that $\KL(p||q) = \int p(x) \ln \frac{p(x)}{q(x)} dx$. If $q(x)=0$ then $p(x)=0$ otherwise the distance is infinite. Second if $p(x)=0$, then the contribution of $q(x)$ is not considered since the integral is taken of the support of $p$, thus, since the intersection is not null and $p$ is the result of an optimization, the support of $p$ is the intersection of the support of $q$ and the support of $P$.   
\end{proof}

\subsection{Datasets and color correction}
We here visualize few of the dataset and color correlations. Figure~\ref{fig:mnist_train_test_b11} shows Fashion-MNIST and the b11 color correlation. In the test environment the background color of each class is inverted. In Figure~\ref{fig:mnist_and_fashionmnist_origina} we show the dataset as generated from \cite{ahuja2020invariant}. In Figure~\ref{fig:emnist_and_kmnist} we show the EMINST (letter) and KMNIST dataset.

% \begin{figure}[!hbpt]
% 	\centering
% 	\subfigure[] {
% 		\includegraphics[width=0.4\textwidth]{dataset_pt_FashionMNIST_type_b11_train_ver_vis.pdf}
% 	}
% 	\subfigure[] {
% 		\includegraphics[width=0.4\textwidth]{dataset_pt_FashionMNIST_type_b11_test_ver_vis.pdf}
% 	}
% 	\caption{Fashion MNIST dataset training (a) and testing (b) environments; the color is inverted as an example; in this case the background changes color}
% 	\label{fig:mnist_train_test_b11}
% \end{figure}

\begin{figure}[!hbpt]
	\centering
	\subfigure[] {
		\includegraphics[width=0.4\textwidth, trim = 0 0 0 .7cm,clip]{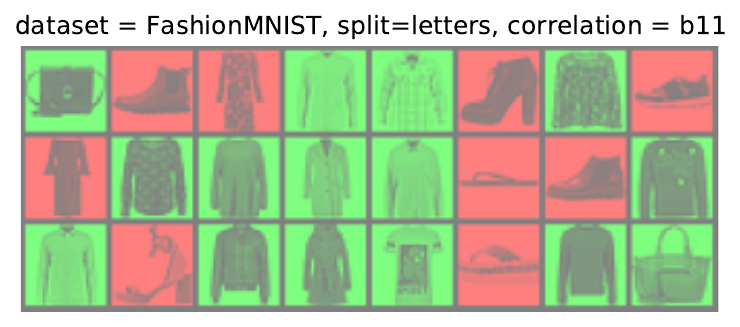}
	}
	\subfigure[] {
		\includegraphics[width=0.4\textwidth, trim = 0 0 0 .7cm,clip]{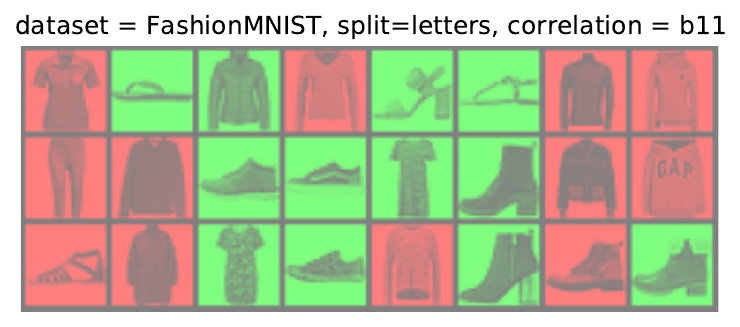}
	}
	\caption{Fashion MNIST dataset training (a) and testing (b) environments; the color is inverted based on the b11 color correlation scheme, where the background color depends on the class of the image. In the test environment the dependency is inverted.}
	\label{fig:mnist_train_test_b11}
\end{figure}

\begin{figure}[!hbpt]
	\centering
	\subfigure[] {
	\includegraphics[width=0.4\textwidth]{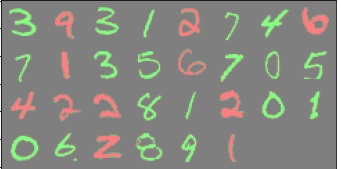}
% 	\caption{MNIST dataset where the is generate by the label}
% 	\label{fig:mnist}
	}
% \end{figure}
% \begin{figure}[!hbpt]
% 	\centering
    \subfigure[] {
	\includegraphics[width=0.4\textwidth]{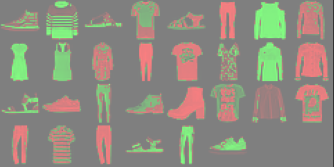}
% 	\caption{MNIST FASHION dataset where the is generate by the label}
% 	\label{fig:mnist_fashion}
	}
	\caption{MNIST dataset (a) and Fashion MNIST (b) environments as defined in \cite{ahuja2020invariant}}
	\label{fig:mnist_and_fashionmnist_origina}
\end{figure}

% \begin{figure}[!hbpt]
% 	\centering
% 	\includegraphics[width=0.5\textwidth]{dataset_mnist_fashion_type_b10.pdf}
% 	\caption{MNIST FASHION dataset where the is generate by the label}
% 	\label{fig:mnist_fashion_b10}
% \end{figure}

% \begin{figure}[!hbpt]
% 	\centering
% 	\includegraphics[width=0.5\textwidth]{dataset_pt_MNIST_type_b11_train_ver_vis.pdf}
% 	\caption{MNIST  dataset where the is generate by the label}
% 	\label{fig:mnist_b11}
% \end{figure}

\begin{figure}[!hbpt]
	\centering
	\subfigure[] {
`	\includegraphics[width=0.45\textwidth, trim = 0 0 0 .7cm,clip]{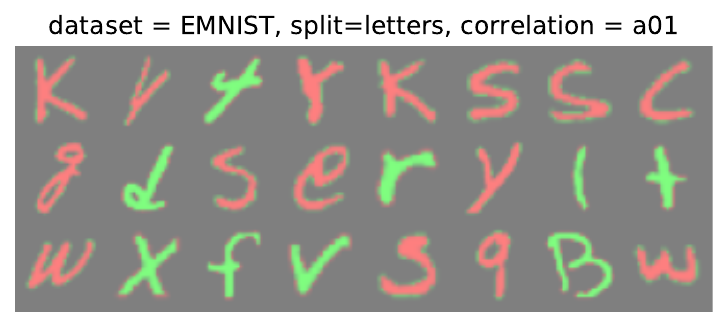}
% 	\caption{EMNIST (letters)  dataset }
% 	\label{fig:emnist_letter_a011}
% \end{figure}
    }
    \subfigure[] {
% \begin{figure}[!hbpt]
% 	\centering
	\includegraphics[width=0.45\textwidth, trim = 0 0 0 .7cm,clip]{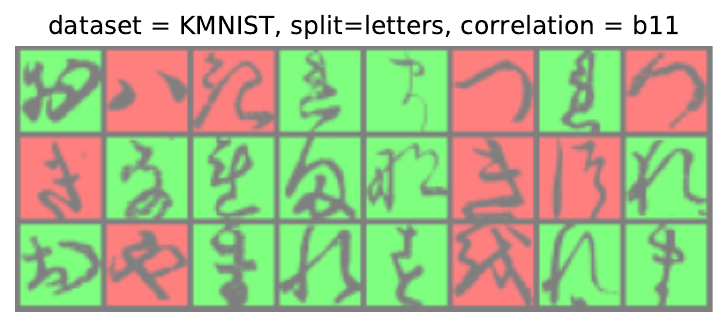}
% 	\caption{EMNIST  dataset }
% 	\label{fig:kmnist_b11}
    }
    \caption{Examples of the EMNIST dataset (a) and of the KMNIST (b).}
    \label{fig:emnist_and_kmnist}
\end{figure}

\subsection{Hyper-parameter Search and Experimental Setup}
We performed hyper-parameter search around the suggested values from the original works and the values selected based on the best performance on the test environment. To implement a complete comparison we used for training $1'000$ samples randomly drawn from each environment. All methods were trained on the same data, using random seed reset. We trained all method with $100$ epochs on a batch size of $256$. 
\begin{itemize}
    \item \textbf{IRM}: $\gamma=91257$,  threshold = 1/2 epochs, learning rate $2.5e^{-4}$
    \item \textbf{IRMG}:  warm start=$300$, termination accuracy $0.6$, learning rate $2.5e^{-4}$, dropout probability $75\%$, weight decays = $.00125$
    \item \textbf{ERM}: learning rate $1e^{-3}$, dropout probability $75\%$, weight decays = $.00125$
    \item \textbf{MER}: memory size $100$ ($10\%$ of the samples), learning rate $1e^{-3}$, replay batch size =$5$, $\beta=.03$, $\gamma=1.0$
    \item \textbf{GEM}: memory size $100$ ($10\%$ of the samples), learning rate $1e^{-3}$
    \item \textbf{EWC}: memory size $100$ ($10\%$ of the samples), learning rate $1e^{-3}$, regularization $0.1$
    % \item \textbf{VCL*}: weight decays = $.00125$, $\lambda = 1.$, number  evaluations $5$ 
    \item \textbf{VCL,VCLC}: learning rate = $5e^{-3}$, corset size $100$ ($10\%$ of the samples),
        
    \item \textbf{C-BVIRM, C-VIRMG, C-VIRMv1}: weight decays = $.00125$, $\beta = 1.$, number  evaluations $5$, $\rho_0=\rho_1=10$, step threshold =1/2 epochs, $\delta \rho = 100$,  learning rate $1e^{-3}$ 
\end{itemize}
The neural network architecture is composed of $2$ non-linear Exponential Linear Unit (ELU) activated Full Connected layers of size $100$, followed by a linear full connected layer. Each layer with dropout. Dropout is not present in VCL/VCLC since not implemented in the original work. Training loss is the Cross Entropy. We tested also with the feature extraction layer separated, but with no advantage, since the test set-up only consist of one task. 

{\color{black}
The IRMv1, IRMG and ERM methods, similarly to the other methods, are trained sequentially as data from each new environment arrives. The Continual Learning methods are allowed to have a limited memory of samples from previous environments. 
}

%\subsection{New Expriments}

%\subsubsection{EIIL}
%We explore the use of the Variational Model with the Environment classification step proposed in \cite{creager2020environment}. We notice as the Variational model helps especially after the Environment Inference for Invariant Learning (EIIL) classification method proposed in \cite{creager2020environment}.

{\color{black}
\subsection{Hyper-parameters of environment inference for continual invariant learning}
We list below the values of hyper-parameters in EIIL for continual invariant learning: \\
$2$ layers, \\
$390$ hidden neurons,\\
$501$ epochs,\\
l2 regularizer weight: $0.00110794568$,\\
learning rate:  $0.0004898536566546834$,\\
numbber of runs: $10$,\\
number of EIIL iterations: $10'000$,\\
number Monte Carlo evaluations: $3$,\\
penalty anneal iterations: $190$,\\
penalty weight: $191257.18613115902$,\\
prior weight: $1e^{-6}$

\subsubsection{Synthetic Dataset}
\begin{table}
\centering
	\caption{Mean accuracy (over $5$ runs) on Synthetic Dataset (\cite{arjovsky2019invariant},\cite{creager2020environment}). BIRM refers to our bilevel objective Eq. (8) optimized with ADMM.}
	\label{tab:SYNTH_EIIL}
%sorry i am done	
% 	\small
\footnotesize
{\color{black}
\begin{tabular}{rll}
\toprule
& Causal MSE & Noncausal MSE  \\
\midrule
ERM	& $0.827   \pm  0.016  $ &	$ 0.824   \pm 	 0.015 $  \\
ICP	& $1.000  \pm   0.000  $  & 	$ 0.756   	 \pm  0.423 $  \\
IRM	& $0.280   	\pm  0.006  $ &	 $0.290   \pm 	 0.009  $ \\
\textbf{BIRM} &	 $0.183   \pm 	 0.005 $ & 	$ {\bf 0.184 }   \pm 	 0.002  $ \\
EIIL(IRM) &	  ${\bf0.180}  \pm  	 0.026 $  & $	 0.188   \pm 	 0.033 $  \\
%   \hline
 \bottomrule
\end{tabular}
}
\end{table}
The Synthetic Dataset is described in (\cite{arjovsky2019invariant}) for testing IRM and it is defined by a Structural Causal Model (\cite{pearl2009causality}), where a variable $y \in \R^{N}$ is generated by $x_1 \in \R^{N}$, while $x_2 \in \R^{N}$ is generated by $y$. The observed variable is $x = (x_1,x_2)$. The structural equations are 
\begin{eqnarray}
    x_1 &=& \epsilon_1,~ \epsilon_1 \sim N(0,\sigma_1^2) \\
    y &=& x_1+\epsilon_y, ~ \epsilon_y \sim N(0,\sigma_e^2) \\
    x_1 &=& y + \epsilon_2,~ \epsilon_2 \sim N(0,1) 
\end{eqnarray}
with $\sigma_1$ fixed and $\sigma_e$ dependent on the environment. We compared with ERM, IRM (\cite{arjovsky2019invariant}), IPC (Invariant Prediction), which is the method proposed in \cite{peters2015causal}, and EIIL ( \cite{creager2020environment}). We use a similar set up of \cite{creager2020environment}, with $N=4$. The invariant model is given by $w = (1,0)$.

}

\end{document}